\documentclass{article}

\usepackage{microtype}
\usepackage{graphicx}
\usepackage{subfigure}
\usepackage{booktabs} 

\usepackage{hyperref}
\usepackage{url}

\usepackage[accepted]{icml2024}

\usepackage{amsmath}
\usepackage{amssymb}
\usepackage{mathtools}
\usepackage{amsthm}
\usepackage{algorithm}
\usepackage{algorithmic}
\usepackage{subfigure}
\usepackage{multirow}
\usepackage{calc}
\usepackage{enumitem,calc}
\usepackage{lipsum}
\usepackage{thmtools, thm-restate}
\usepackage{pifont}
\usepackage{wrapfig}

\usepackage[capitalize,noabbrev]{cleveref}

\theoremstyle{plain}

\newtheorem{problem}{Problem}

\newtheorem{lemma}{Lemma}

\newtheorem{definition}{Definition}
\newtheorem{assumption}{Assumption}
\newcommand\preitem{\mdseries\textbullet\space}
\usepackage{enumitem,calc}
\newlist{desclist}{description}{3}
\setlist[desclist,1]{format=\preitem\bfseries,leftmargin=\widthof{\preitem},style=sameline}

\usepackage{expl3}
\ExplSyntaxOn
\newcommand\latinabbrev[1]{
  \peek_meaning:NTF . {
    #1\@}%
  { \peek_catcode:NTF a {
      #1.\@ }%
    {#1.\@}}}
\ExplSyntaxOff
\def\eg{\latinabbrev{e.g}}
\def\etal{\latinabbrev{et al}}

\def\ie{\latinabbrev{i.e}}
\newcommand{\fname}{{\textsc{EvoluNet}}}

\usepackage[textsize=tiny]{todonotes}

\begin{document}

\twocolumn[
\icmltitle{EvoluNet: Advancing Dynamic Non-IID Transfer Learning on Graphs}



\icmlsetsymbol{equal}{*}

\begin{icmlauthorlist}
\icmlauthor{Haohui Wang}{VT}
\icmlauthor{Yuzhen Mao}{VT}
\icmlauthor{Yujun Yan}{Dartmouth}
\icmlauthor{Yaoqing Yang}{Dartmouth}
\icmlauthor{Jianhui Sun}{VT}
\icmlauthor{Kevin Choi}{Deloitte}
\icmlauthor{Balaji Veeramani}{Deloitte}
\icmlauthor{Alison Hu}{Deloitte}
\icmlauthor{Edward Bowen}{Deloitte}
\icmlauthor{Tyler Cody}{VTNSI}
\icmlauthor{Dawei Zhou}{VT}
\end{icmlauthorlist}

\icmlaffiliation{VT}{Department of Computer Science, Virginia Tech, Blacksburg, VA, USA.}
\icmlaffiliation{VTNSI}{Virginia Tech National Security Institute, Arlington, VA, USA}
\icmlaffiliation{Dartmouth}{Department of Computer Science, Dartmouth College, Hanover, NH, USA.}
\icmlaffiliation{Deloitte}{Deloitte \& Touche LLP, USA.}

\icmlcorrespondingauthor{Dawei Zhou}{zhoud@vt.edu}

\icmlkeywords{Machine Learning, ICML}

\vskip 0.3in
]



\printAffiliationsAndNotice{}  

\begin{abstract}
Non-IID transfer learning on graphs is crucial in many high-stakes domains. The majority of existing works assume stationary distribution for both source and target domains. However, real-world graphs are intrinsically dynamic, presenting challenges in terms of domain evolution and dynamic discrepancy between source and target domains. 
To bridge the gap, we shift the problem to the dynamic setting and pose the question: {given the \emph{label-rich} source graphs and the \emph{label-scarce} target graphs both observed in previous $T$ timestamps, how can we effectively characterize the evolving domain discrepancy and optimize the generalization performance of the target domain at the incoming $T+1$ timestamp?}
To answer it, we propose a generalization bound for \emph{dynamic non-IID transfer learning on graphs}, which implies the generalization performance is dominated by domain evolution and domain discrepancy between source and target graphs. 
Inspired by the theoretical results, we introduce a novel generic framework named \fname. It leverages a transformer-based temporal encoding module to model temporal information of the evolving domains and then uses a dynamic domain unification module to efficiently learn domain-invariant representations across the source and target domains. 
Finally, \fname\ outperforms the state-of-the-art models by up to 12.1\%, demonstrating its effectiveness in transferring knowledge from dynamic source graphs to dynamic target graphs. 
\end{abstract}

\section{Introduction}
The recent decade has witnessed notable achievements in machine learning. Despite the exciting achievements, whether the learned model could deliver its promise in real-world scenarios heavily depends on abundant and high-quality training data.
Nevertheless, the data annotation process, which requires domain-specific knowledge from human annotators, is often a costly and time-intensive endeavor~\cite{Fang2021Molecular, Cui2022Interpretable,DBLP:conf/cikm/ZhouZF0H22}. Transfer learning has emerged as a promising tool, which aims to improve the generalization performance of the target domain with little or no labeled data by leveraging knowledge from the source domain with adequate labeled data~\cite{Tripuraneni20Theory, Wang19Characterizing, Ganin16Domain, Ben06Advances}. 
However, a majority of the existing work~\cite{Ben10Theory, zhao2019learning, Wang22Understanding} hold the assumption that data is static and independent and identically distributed (IID), as shown in the bottom-left of Figure~\ref{fig:paradigm}.
When extending transfer learning to graph-structured data, particular challenges are posed due to the non-IID nature of graphs, i.e., samples on graphs (e.g., nodes, edges, subgraphs) are naturally connected with their neighbors in certain ways~\cite{xie2021federated}.

\begin{figure}[t]
  \centering
  \includegraphics[width=1.0\linewidth]{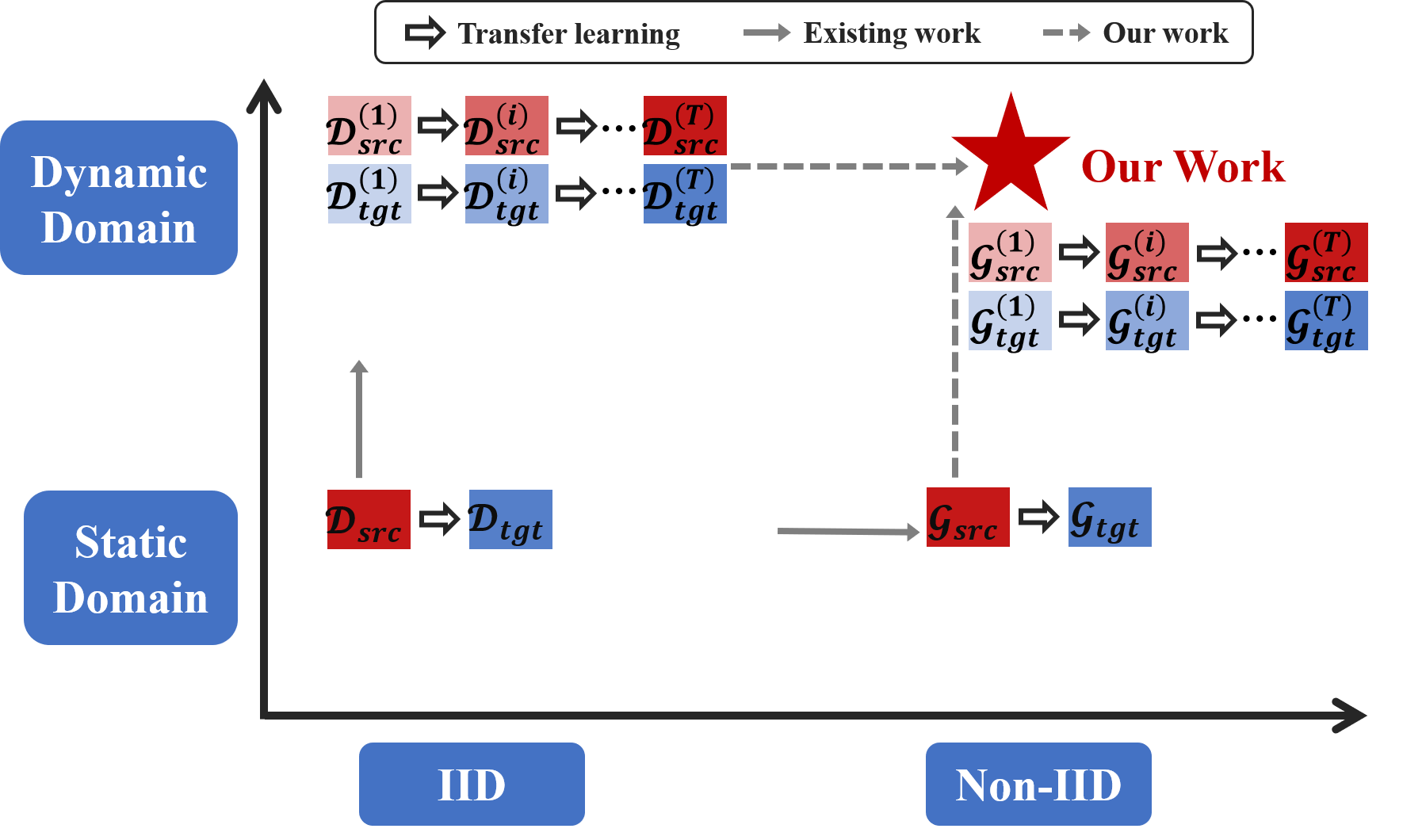}
  \begin{footnotesize}
  \put(-181,30){\cite{Ben10Theory}}
  \put(-181,40){\cite{Wang22Understanding}}
  \put(-80,40){\cite{wu2023non}}
  \put(-80,30){\cite{zhu2021transfer}}
  \put(-181,95){\cite{Wu22Unified}}
  \put(-181,85){\cite{Liu20Learning}}
  \end{footnotesize}
  \caption{A paradigm shift to the dynamic non-IID transfer learning. $\mathcal{D}$ denotes IID domain; $\mathcal{G}$ denotes non-IID graph domain. Subscript $src$ and $tgt$ denote source and target, and superscript $(i)$ represents the $i^\text{th}$ timestamp.}
  \label{fig:paradigm}
\end{figure}

While recent research efforts have delved into non-IID transfer learning on graphs~\cite{wu2023non, zhu2021transfer, hu2020graph, wu2020unsupervised, shen2020adversarial} as shown in the bottom-right of Figure~\ref{fig:paradigm}, the most of them have overlooked the dynamics inherent in realistic systems, where graphs in both source and target domains evolve over time. 
Directly applying existing static works on graphs to the dynamic setting may lead to a sub-optimal performance due to the unexplored temporal information and evolving distribution discrepancy~\cite{Greene10Tracking,de2014graph,Pareja20EvolveGCN,Song19Session,DBLP:conf/cikm/FuXLTH20,DBLP:conf/kdd/ZhouZ0H20}. 
Therefore, our paper proposes a paradigm shift in Figure~\ref{fig:paradigm} towards the dynamic non-IID setting, by introducing the novel problem termed as \emph{dynamic non-IID transfer learning on graphs}. In particular, given the \emph{label-rich} source graphs and the \emph{label-scarce} target graphs observed in previous $T$ timestamps, how can we effectively characterize the evolving domain discrepancy and optimize the generalization performance of the target graph at the incoming $T+1$ timestamp?

Despite the key importance, there exist three pivotal challenges in our problem setting.
\emph{C1. Generalization Bound}: There is limited theoretical analysis on how the domain discrepancy would accumulate across time and how it will affect the model performance. 
Carrying out theoretical analysis on the generalization bound would be crucial for understanding dynamic non-IID transfer learning on graphs. 
\emph{C2. Computational Framework}: How can we develop a computational framework to characterize the evolving domain discrepancy and capture the domain-invariant information when the source and target graphs exhibit distinct distributions over time? 
\emph{C3. Benchmark}: 
As there is little existing literature on dynamic non-IID transfer learning on graphs, it is essential to point out a set of benchmark datasets and baselines for algorithm development and evaluation. 

In this paper, we make the first attempt to derive a generalization bound for dynamic transfer learning on graphs. The theoretical findings illustrate that the generalization performance is dominated by historical empirical error and domain discrepancy. It also serves as theoretical support to our proposed \fname, which is a generic learning framework to enhance knowledge transfer across dynamic graphs.
Moreover, we utilize a multi-resolution temporal encoding module to model domain evolution and a module to minimize domain discrepancy via dual divergence loss. In particular, the first module captures the interdependence over
time and obtains the temporal graph representation in the evolving graphs, while the second module learns invariant representations to unify the source and target domains' spatial and temporal information.
Our empirical results show that \fname\ outperforms the state-of-the-art models by up to 12.1\%, underscoring its effectiveness in knowledge transfer across dynamic graphs.
Furthermore, we extensively surveyed existing temporal graphs and constructed benchmark datasets\footnote{We publish our data and code at~\url{https://github.com/wanghh7/EvoluNet}.} for dynamic non-IID transfer learning, which have rich, dynamic properties regarding nodes, edges, node attributes, and labels. We conduct various evaluations on the constructed benchmark dataset, which demonstrate its validity and reliance.

\section{Preliminary}\label{sec:preliminary}
In this section, we introduce the background that is pertinent to our work and give the formal problem definition. Table~\ref{TB:Notations} summarizes the main notations used in this paper. We use regular letters to denote scalars (\eg, $\mu$), boldface lowercase letters to denote vectors (\eg, $\mathbf{v}$), and boldface uppercase letters to denote matrices (\eg, $\mathbf{X}$).
Next, we briefly review non-IID transfer learning on graphs and dynamic transfer learning for IID distributions.\\

\begin{table}[h]
\caption{Symbols and notations.}
\centering
\scalebox{0.95}{
\begin{tabular}{|l|l|}
\hline Symbol&Description\\
\hline
\hline 
$\mathcal{G}_{src}^{(i)}$, $\mathcal{G}_{tgt}^{(i)}$&input source and target graphs at timestamp $i$.\\
$\mathcal{V}_{src}^{(i)}$, $\mathcal{V}_{tgt}^{(i)}$&the set of nodes in $\mathcal{G}_{src}^{(i)}$ and $\mathcal{G}_{tgt}^{(i)}$.\\
$\mathcal{E}_{src}^{(i)}$, $\mathcal{E}_{tgt}^{(i)}$&the set of edges in $\mathcal{G}_{src}^{(i)}$ and $\mathcal{G}_{tgt}^{(i)}$.\\
$\mathbf{X}_{src}^{(i)}$, $\mathbf{X}_{tgt}^{(i)}$&the node feature matrices of $\mathcal{G}_{src}^{(i)}$, $\mathcal{G}_{tgt}^{(i)}$.\\
$\mathcal{Y}_{src}^{(i)}$, $\tilde{\mathcal{Y}}_{tgt}^{(i)}$&the set of labels in $\mathcal{G}_{src}^{(i)}$ and $\mathcal{G}_{tgt}^{(i)}$.\\
$N_{src}^{(i)}$, $N_{tgt}^{(i)}$&the size of sample graph $\mathcal{G}_{src}^{(i)}$, $\mathcal{G}_{tgt}^{(i)}$.\\
$d_{src}$, $d_{tgt}$&feature dimensions of $\mathbf{X}_{src}^{(i)}$, $\mathbf{X}_{tgt}^{(i)}$, $\forall i$.\\
$T$&number of timestamps.\\
\hline 
$h(\cdot)$&node classifier for downstream task.\\
\hline 
$\tilde{\Re}$&Rademacher complexity.\\
$W_p$&$p$-Wasserstein distance.\\
\hline
\end{tabular}
}
\label{TB:Notations}
\end{table}

\noindent\textbf{Non-IID Transfer Learning on Graphs. }It focuses on leveraging knowledge gained from a source graph $\mathcal{G}_{src}$ to improve the performance of a target graph $\mathcal{G}_{tgt}$. Graphs are non-IID because their interconnected nodes, edges, and subgraphs exhibit inherent dependencies, and require modeling the highly complex interconnection. To applied existing theoretical guarantees of transfer learning under IID assumption to non-IID graph data, Wu \etal~\yrcite{wu2023non} propose a novel graph discrepancy $d_{\text{GSD}}(\mathcal{G}_{src}, \mathcal{G}_{tgt})$ between two graphs $\mathcal{G}_{src}$ and $\mathcal{G}_{tgt}$ as follows (informal):
\begin{equation*}
    d_{\text{GSD}}(\mathcal{G}_{src}, \mathcal{G}_{tgt})=\lim _{M \rightarrow \infty} \frac{1}{M+1} \sum_{m=0}^M d_b(\mathcal{G}_{src}^m, \mathcal{G}_{tgt}^m),
\end{equation*}
where $\mathcal{G}^m$ is the Weisfeiler-Lehman subgraph~\cite{Shervashidze11Weisfeiler} at depth $m$ for an input graph $\mathcal{G}$, $d_b(\cdot,\cdot)$ is the base domain discrepancy. We refer to Definition~\ref{def:GraphDis} in Appendix~\ref{sec:proof} for formal definition.
It reveals that given Weisfeiler-Lehman subtree, the subtree representations can be considered as IID samples, thus existing distribution discrepancy measures (\eg, Maximum Mean Discrepancy (MMD)~\cite{Gretton12kernel} and Wasserstein Distance~\cite{villani09optimal}) can be used to measure the distribution shift of source and target graphs.

\noindent\textbf{Dynamic Transfer Learning.}
Let $\{\mathcal{D}_{src}^{(i)}\}_{i=1}^T$ and $\{\mathcal{D}_{tgt}^{(i)}\}_{i=1}^T$ be the labeled dynamic source domains and unlabeled (or few labeled) dynamic target domains, where superscript $(i)$ represents the $i^\text{th}$ timestamp, and there are $T$ total timestamps. Dynamic transfer learning aims to improve the prediction performance of
$\mathcal{D}_{tgt}^{(T+1)}$ using the knowledge in historical source and target domains under the domain shift $\{\mathcal{D}_{src}^{(i)}\}_{i=1}^T \neq \{\mathcal{D}_{tgt}^{(i)}\}_{i=1}^T$. Let $\mathcal{H}$ be the hypothesis class on input feature space $\mathcal{X}$ where a hypothesis is a function $h: \mathcal{X} \rightarrow \mathcal{Y}$, and $\mathcal{Y}$ is output label space. The expected error of the hypothesis $h$ on the source domain $\mathcal{D}_{src}^{(i)}$ at timestamp $i$ is given by $\epsilon_{src}^{(i)}(h)=\mathbb{E}_{\mathbf{x} \sim \mathcal{D}_{src}^{(i)}}[\mathcal{L}(h(\mathbf{x}), y)], \forall h \in \mathcal{H}$, where $\mathcal{L}(\cdot,\cdot)$ is some loss function. Its empirical estimate is defined as $\hat{\epsilon}_{src}^{(i)}(h)=\frac{1}{N_{src}^{(i)}}\sum_{j=1}^{N_{src}^{(i)}}[\mathcal{L}(h(\mathbf{x}_j), y_j)]$, where $\mathbf{x}_j$ is the feature of $j^\text{th}$ sample in  $\mathbf{X}_{src}^{(i)}$. We use the parallel notations $\epsilon_{tgt}^{(i)}(h)$ and $\hat{\epsilon}_{tgt}^{(i)}(h)$ for the target domain. 
In~\citet{Wu22Unified}, the expected error on the newest target domain is derived as follows:
\begin{equation}
\begin{aligned}
\epsilon_{tgt}^{(T+1)}(h) & \leq \frac{1}{2 T} \sum_{i=1}^{T}\left(\hat{\epsilon}_{src}^{(i)}(h)+\hat{\epsilon}_{tgt}^{(i)}(h)\right)+\frac{T+2}{2}(\tilde{d}+\tilde{\lambda}) \\
&+\tilde{\Re}\left(\mathcal{H}_{\mathcal{L}}\right)+\frac{\rho}{T} \sqrt{\frac{\log \frac{1}{\delta}}{2 \tilde{m}}}, \\
\end{aligned}
\end{equation}
where $\tilde{d}=\rho \cdot \max \left\{\max _{1 \leq i \leq T-1} d_{\text{MMD}}\left(\mathcal{D}^{(i)}_{src}, \mathcal{D}^{(i+1)}_{src}\right), \right. \\
\left. d_{\text{MMD}}\left(\mathcal{D}^{(1)}_{src}, \mathcal{D}^{(1)}_{tgt}\right), \max _{1 \leq i \leq T} d_{\text{MMD}}\left(\mathcal{D}^{(i)}_{tgt}, \mathcal{D}^{(i+1)}_{tgt}\right)\right\}$, $d_{\text{MMD}}$ is the maximum mean discrepancy~\cite{Gretton12kernel}, $\rho$ is the Lipschitz constant, $\tilde{\lambda}=\rho \cdot \max\left\{\max _{1 \leq i \leq T-1}\lambda_{*}\left(\mathcal{D}^{(i)}_{src}, \mathcal{D}^{(i+1)}_{src}\right), \lambda_{*}\left(\mathcal{D}^{(1)}_{src}, \mathcal{D}^{(1)}_{tgt}\right), \right. \\ \left. \max _{1 \leq i \leq T} \lambda_{*}\left(\mathcal{D}^{(i)}_{tgt}, \mathcal{D}^{(i+1)}_{tgt}\right)\right\}$, $\lambda_{*}$ measures the labeling difference.
$\mathcal{H}_{\mathcal{L}}=\{(\mathbf{X}, y)\mapsto\mathcal{L}(h(\mathbf{X}), y): h\in\mathcal{H}\}$, $\tilde{\Re}\left(\mathcal{H}_{\mathcal{L}}\right)$ is a term that involves the Rademacher complexity defined on multiple domains, and $\tilde{m}=\sum_{i=1}^{T}\left(N^{(i)}_{src}+N^{(i)}_{tgt}\right)$ is the total number of training examples from historical source and target domains. However, this bound sums the errors in all timestamps without capturing domain evolution.

\begin{figure}[t]
  \centering
  \includegraphics[width=1.0\linewidth]{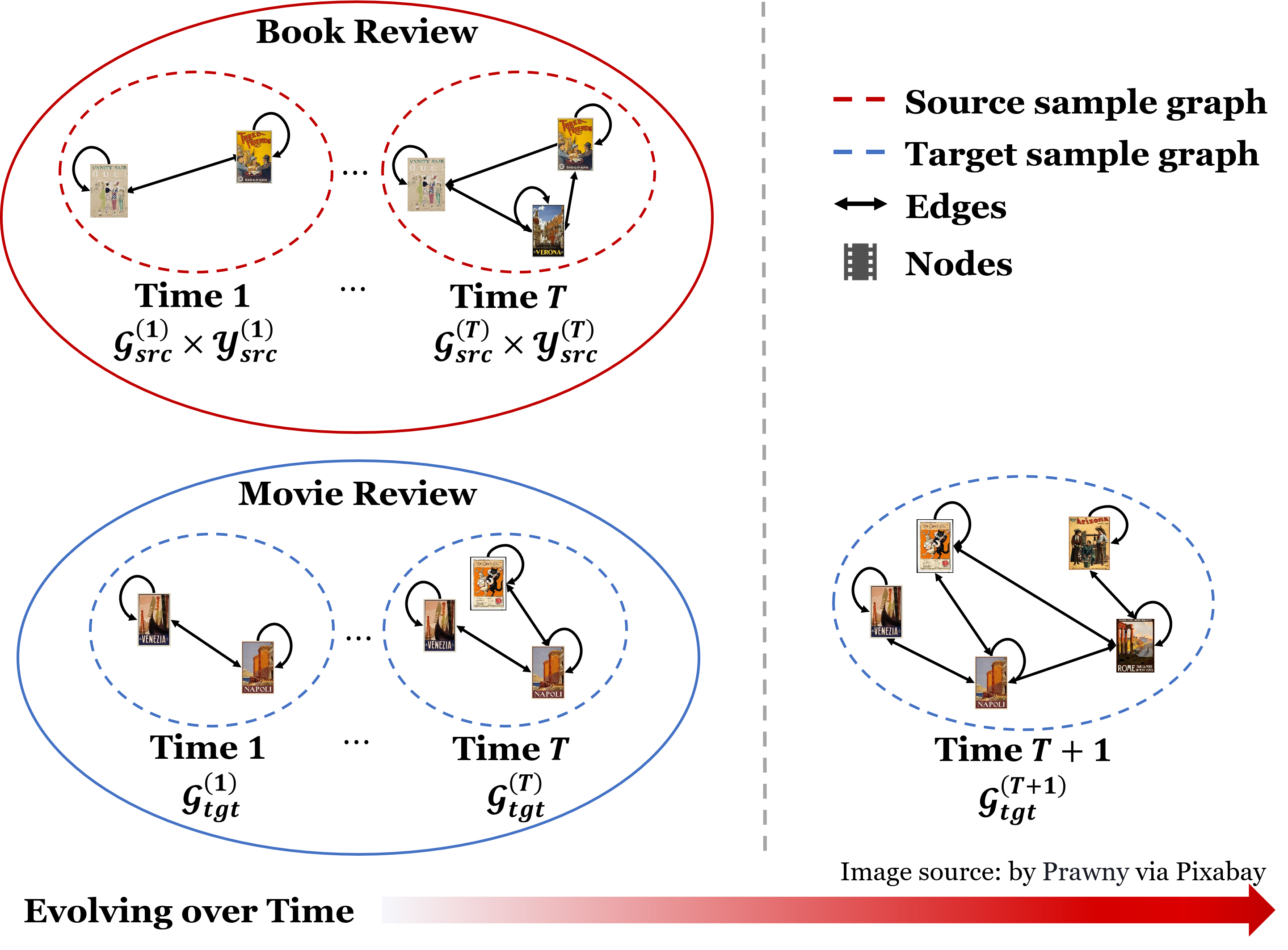}
  \caption{An illustrative example of dynamic non-IID transfer learning on book review graph and movie review graph. As an example, consider a new series launched on a movie website, where the original book of this series may have been published for decades. It is very natural to transfer knowledge from the information-rich source domain (book) to the information-scarce target domain (movie) across time in order to solve the target task (movie review prediction) at $\mathcal{G}_{tgt}^{(T+1)}$.} 
  \label{fig:example}
\end{figure}

\noindent\textbf{Problem definition.} In the setting of dynamic non-IID transfer learning on graphs,
the observed graph in source at timestamp $i$ is defined as a source sample graph $\mathcal{G}_{src}^{(i)} = (\mathcal{V}_{src}^{(i)}, \mathcal{E}_{src}^{(i)})$ (parallel definition of target sample graph $\mathcal{G}_{tgt}^{(i)}=(\mathcal{V}_{tgt}^{(i)}, \mathcal{E}_{tgt}^{(i)})$), where $\mathcal{V}_{src}^{(i)}$ and $\mathcal{V}_{tgt}^{(i)}$ represent the set of nodes, and $\mathcal{E}_{src}^{(i)}$ and $\mathcal{E}_{tgt}^{(i)}$ represent the set of edges, respectively. 
$\mathbf{X}_{src}^{(i)}$ and $\mathbf{X}_{tgt}^{(i)}$ represent the node features in $\mathcal{G}_{src}^{(i)}$ and $\mathcal{G}_{tgt}^{(i)}$,
$T$ is the total number of timestamps that can be observed in history. 
Furthermore, we define labels of source at $i^{th}$ timestamp as $\mathcal{Y}_{src}^{(i)}$ and labels of target as $\tilde{\mathcal{Y}}_{tgt}^{(i)}$, where only few target samples have labels, so $\tilde{\mathcal{ Y}}_{tgt}^{(i)}$ is a sparse vector. 
We consider transferring knowledge from a series of time-evolving source sample graphs $\{\mathcal{G}_{src}^{(i)}\}_{i=1}^T$ to a series of time-evolving target sample graphs $\{\mathcal{G}_{tgt}^{(i)}\}_{i=1}^{T+1}$. Figure~\ref{fig:example} illustrates the knowledge transfer from historical time snapshots of the book review graph and movie review graph to a more recent time snapshot of the movie review graph. Here, each node in a graph indicates an entity (user, movie, book), and the co-reviewer determines the edge between two nodes. Nodes, edges, and their attributes are evolving over time. The node label is the popularity of the movie (book) at that time and is also changing.

Given the notations above, we formally define the problem as follows.
\begin{problem}\label{prob}
  \textbf{Dynamic Non-IID Transfer Learning on Graphs}\\
  \textbf{Given:} (i) a set of source sample graphs $\{\mathcal{G}_{src}^{(i)} = (\mathcal{V}_{src}^{(i)}, \mathcal{E}_{src}^{(i)})\}_{i=1}^T$ with rich label information $\{\mathcal{Y}_{src}^{(i)}\}_{i=1}^T$, and (ii) a set of target sample graphs $\{\mathcal{G}_{tgt}^{(i)} = (\mathcal{V}_{tgt}^{(i)}, \mathcal{E}_{tgt}^{(i)})\}_{i=1}^{T+1}$ with few label information $\{\tilde{\mathcal{Y}}_{tgt}^{(i)}\}_{i=1}^{T+1}$.\\
  \textbf{Find:} Accurate predictions $\hat{\mathcal{Y}}_{tgt}^{(T+1)}$ of unlabeled examples in the target sample graph $\mathcal{G}_{tgt}^{(T+1)}$.
\end{problem}

\section{Model}\label{sec:model}
In this section, we introduce our proposed framework \fname\ for dynamic non-IID transfer learning on graphs. 
The key idea lies in regularizing the underlying evolving domain discrepancy, which mainly stems from the distribution shift due to domain evolution and the inherent domain discrepancy between the source and target domains.
In particular, we start with deriving a novel generalization bound of Problem~\ref{prob}, which is composed of historical empirical errors on the source and target domains, domain discrepancies across time on source and target, and Rademacher complexity of the hypothesis class. 
Inspired by the theoretical results, we then develop the overall learning paradigm of \fname\ and discuss the details of how to model domain evolution and how to unify dynamic graph distribution. 
Finally, we present an optimization algorithm with pseudo-code for \fname\ in Algorithm~\ref{Alg} in Appendix~\ref{sec:pseudo}.

\subsection{Theoretical Analysis}\label{theory}
Here, we propose the very first generalization guarantee under the setting of dynamic non-IID transfer learning on graphs. The existing literature~\cite{Wu22Unified} leads to a  loosely bound in special cases when the historical empirical error at a specific timestamp is extremely large. This is because Wu and He's work simply accumulates the empirical errors across time, which results in their generalization error bound being sensitive to extreme cases.

To derive a better error bound, we propose to improve our bound mainly from the following three aspects: 
(1) We propose to replace $\sum_{i=1}^{T}\left(\hat{\epsilon}_{src}^{(i)}(h)+\hat{\epsilon}_{tgt}^{(i)}(h)\right)$ with the minimum value of historical empirical errors on source and target. 
The conventional measurement is notably susceptible to outliers over time. This sensitivity becomes particularly evident when a machine learning model encounters failures at specific timestamps, leading to exceptionally large empirical errors in both the source and target domains. In contrast, using the minimum value demonstrates inherent resilience to such extreme cases and prevents the error bound from being impacted by some extreme cases.
(2) We develop a novel dynamic Wasserstein distance to replace maximum mean discrepancy $\tilde{d}$ for better measuring the evolving domain discrepancy. (3) To accurately characterize the graph distribution shift, we propose to construct the Weisfeiler-Lehman subgraphs at each timestamp and then compute the dynamic graph discrepancy upon them. 

We first introduce the definition of dynamic $p$-Wasserstein distance on graphs, which measures the graph discrepancy across tasks and across time stamps.
\begin{restatable}[Dynamic $p$-Wasserstein Distance on Graphs]{definition}{wassersteinG}
Consider two dynamic graphs $\{\mathcal{G}_{src}^{(i)}\}_{i=1}^T$ and $\{\mathcal{G}_{tgt}^{(i)}\}_{i=1}^{T+1}$. For any $p\geq 1$, the dynamic $p$-Wasserstein distance is defined as:
\begin{equation*}
\begin{aligned}
    \tilde{W}_{p}=&\rho\sqrt{R^2+1}\max\left(\max_{1\leq i\leq T-1}d_{\text{GSD}}(\mathcal{G}_{src}^{(i)},\mathcal{G}_{src}^{(i+1)}), \right. \\
    & \left. d_{\text{GSD}}(\mathcal{G}_{src}^{(1)}, \mathcal{G}_{tgt}^{(1)}),
    \max_{1\leq i\leq T}d_{\text{GSD}}(\mathcal{G}_{tgt}^{(i)}, \mathcal{G}_{tgt}^{(i+1)})\right),
\end{aligned}
\end{equation*}
where $R$ and $\rho$ are the Lipschitz constants, $d_{\text{GSD}}$ denotes the graph discrepancy based on $p$-Wasserstein distance $W_p$~\cite{wu2023non}.
\end{restatable}

Based on Definition 1, we can apply Lemma~\ref{lemma:shift domains} (Error Difference over Shifted Domains) to bound the error difference on two arbitrary domains as follows.

\begin{restatable}[Error Difference over Shifted Domains~\cite{Wang22Understanding}]{lemma}{shiftDomains}
\label{lemma:shift domains}
For arbitrary classifier $h$ and loss function $\mathcal{L}$ satisfying Assumption~\ref{asmp:R-Lip} and~\ref{asmp:rho-Lip}, the expected error of $h$ on two arbitrary domain $\mathcal{D}_\mu$ and $\mathcal{D}_\nu$ satisfies
\begin{equation*}
    \left|\epsilon_{\mu}(h)-\epsilon_{\nu}(h)\right| \leq \rho \sqrt{R^{2}+1} W_{p}(\mathcal{D}_\mu, \mathcal{D}_\nu),
\end{equation*}
where $W_{p}$ is the $p$-Wasserstein distance metric and $p \geq 1$.
\end{restatable}

Intuitively, Lemma~\ref{lemma:shift domains} yields that the expected error on the target domain at the $N+1$ timestamp $\epsilon_{tgt}^{(T+1)}$ is upper bounded with an expected error on an arbitrary domain and the maximum of measures of domain discrepancy. 
Based on Lemma~\ref{lemma:shift domains}, we can further generalize the difference between the expected error and the empirical error to the arbitrary domains via Lemma~\ref{lemma:stability} (Algorithm Stability) as follows. 
\begin{restatable}[Algorithm Stability, from Lemma A.1 in Kumar \etal~\yrcite{Kumar20Understanding}]{lemma}{stability}
\label{lemma:stability}
With the assumptions~\ref{asmp:R-Lip},~\ref{asmp:rho-Lip},~\ref{asmp:comp}, consider empirical and expected errors on arbitrary domain with $n$ samples, $\forall~\delta \in(0,1)$, the following holds with probability at least $1-\delta$ for some constant $B>0$,
\begin{equation*}
    \left|\hat{\epsilon}(h)-\epsilon(h)\right| \leq \mathcal{O}\left(\frac{\rho B+\sqrt{\log \frac{1}{\delta}}}{\sqrt{n}}\right).
\end{equation*}
\end{restatable}
With Lemma~\ref{lemma:stability}, we are able to bound $\epsilon_{tgt}^{(T+1)}$ with minimal empirical errors on the source and target and the maximum domain discrepancy. That being said, the error of the latest target domain $\epsilon_{tgt}^{(T+1)}$ can be bounded. Finally, we can derive our generalization bound for dynamic non-IID transfer learning on graphs, as stated by the following Theorem~\ref{THM:errorBound}.

\begin{restatable}{theorem}{errorBound}
\label{THM:errorBound}
Assume classifier $h\in\mathcal{H}$ is $R$-Lipschitz and loss function $\mathcal{L}(\cdot, \cdot)$ is $\rho$-Lipschitz, where $R$ and $\rho$ are the Lipschitz constants. For any $\delta>0$, with probability at least $1-\delta$, the error $\epsilon_{tgt}^{(T+1)}$ is bounded by:
\begin{equation}
\begin{aligned}
    \epsilon_{tgt}^{(T+1)}(h) &\leq \frac{1}{2}\min_{1\leq i\leq T}\left(\hat{\epsilon}_{src}^{(i)}(h)+\hat{\epsilon}_{tgt}^{(i)}(h)\right)+\frac{3T}{2}\tilde{W_p}\\
    &+\tilde{\Re}\left(\mathcal{H}_{\mathcal{L}}\right)+\mathcal{O}\left(\frac{\rho B}{\sqrt{\tilde{n}}}+\sqrt{\frac{\log\frac{1}{\delta}}{\tilde{n}}}\right)
\end{aligned}
\end{equation}
where $\tilde{W}_{p}$ is dynamic Wasserstein distance on graphs, $p\geq 1$, 
$\mathcal{H}_{\mathcal{L}}=\{(\mathbf{X}, y)\mapsto\mathcal{L}(h(\mathbf{X}), y): h\in\mathcal{H}\}$, 
$\tilde{\Re}(\mathcal{H}_{\mathcal{L}}) = \frac{1}{2T}\sum_{i=1}^{T}\left(\tilde{\Re}_{\mathcal{D}_{src}^{(i)}}(\mathcal{H}_{\mathcal{L}}) + \tilde{\Re}_{\mathcal{D}_{tgt}^{(i)}}(\mathcal{H}_{\mathcal{L}})\right)$, $\tilde{\Re}$ is Rademacher complexity, $B>0$ is a constant, and $\tilde{n}=\min_{1\leq i\leq T}\left(N_{src}^{(i)}, N_{tgt}^{(i)}\right)$ is the minimal number of training examples in source and target domains.
\end{restatable}
\begin{proof}
The detailed proof is provided in Appendix~\ref{sec:proof}.
\end{proof}
The theorem shows that the error on the latest target domain $\epsilon_{tgt}^{(T+1)}$ is bounded in terms of (1) the minimum value of empirical errors in the historical source and target domains; (2) the maximum of domain discrepancies across time and domain; (3) the average Rademacher complexity of hypothesis class over all domains.

\textbf{Remarks: }Compared to the existing theoretical results on dynamic transfer learning~\cite{Wu22Unified}, we obtain a significantly improved bound in the following aspects. 
\begin{desclist}[topsep=-2mm, itemsep=-1pt]
\item Instead of simply averaging the errors over time as~\cite{Wu22Unified}, we propose to use the minimum of empirical errors over time to imply domain evolution and avoid extreme errors, and we have
\begin{equation*}
\begin{aligned}
    &\min_{1\leq i\leq T}\left(\hat{\epsilon}_{src}^{(i)}(h)+\hat{\epsilon}_{tgt}^{(i)}(h)\right) \\
    \leq &\frac{1}{T}\sum_{1\leq i\leq T}\left(\hat{\epsilon}_{src}^{(i)}(h)+\hat{\epsilon}_{tgt}^{(i)}(h)\right).
\end{aligned}
\end{equation*}
Correspondingly, rather than using an accumulative method to consider all timestamps, \fname\ uses multi-resolution temporal encoding and attention mechanisms to consider domain evolution uniformly.
\item Instead of separately measuring the difference of features and the difference of labels based on MMD, we propose a dynamic Wasserstein distance on graphs to model the evolving graph discrepancy. Correspondingly, \fname\ leverages dual-divergence unification to implicitly reduce this distance~\cite{Ganin16Domain,ganin2015unsupervised}.
\end{desclist}
In general, this generalization bound guarantees the transferability from evolving source domains to evolving target domains and motivates us to propose a framework for dynamic non-IID transfer learning on graphs by empirically minimizing generalization bounds with domain evolution and domain discrepancy.

\subsection{\fname\ Framework}\label{sec:framework}
Without loss of generality, a typical dynamic transfer learning paradigm can be formulated as follows.
\begin{equation}\label{Eq:existing-dynamic-transfer}
    \min_{\theta}\mathcal{L}(\theta) = \sum_{i=1}^{T} \left(\hat{\epsilon}_{src}^{(i)}(\theta) + d(\mathcal{G}^{(i)}_{src}, \mathcal{G}^{(i)}_{tgt}, \theta)\right)
\end{equation}
However, Eq.~\ref{Eq:existing-dynamic-transfer} may not well capture evolving domain discrepancy in practice due to the following two reasons. 
First, Eq.~\ref{Eq:existing-dynamic-transfer} simply sums up the empirical errors over time, which ignores the evolution process of dynamic graphs, \ie, the changes in the future snapshot $\mathcal{G}_{src}^{(t+1)}$ are often highly dependent on the structure of the current snapshot $\mathcal{G}_{src}^{(t)}$~\cite{Seyed20Representation}.
Second, accumulating the domain discrepancy over all timestamps might lose track of the fine-grained information on how domain discrepancies evolve, \eg, the domain discrepancy $d(\mathcal{G}^{(T)}_{src}, \mathcal{G}^{(T)}_{tgt}, \theta)$ in the last timestamp could play a key role in the success of the downstream task in the timestamp $T+1$. 

As shown in Theorem 1, the generalization performance is dominated by two factors: the domain evolution across time and the domain discrepancy on source and target.
Inspired by this, we propose \fname, which consists of two major modules: M1. Modeling Domain Evolution via Multi-Resolution Temporal Encoding and M2. Domain-Invariant Learning via Dual-Divergence Unification. 
In particular, M1 introduces a multi-resolution temporal encoding for dynamic graphs, which encodes temporal information into the representation with continuous values and captures domain evolution by attention; M2 further unifies disparate spatial and temporal information of source and target into the domain-invariant hidden spaces. 
In addition, both M1 and M2 are absolutely necessary to overcome the main obstacles in dynamic non-IID transfer learning on graphs. M1 ensures accurate modeling domain evolution and characterizes historical temporal information for future downstream task-related representation learning, while M2 ensures extraction of domain-invariant spatial and temporal information that could be transferred to benefit the target domain. Our ablation study (Table~\ref{tab:ablation}) firmly attests both M1 and M2 are essential in a successful dynamic graph transfer. The overview of \fname\ is presented in Figure~\ref{fig:framework}.
Next, we dive into the technical details of M1 and M2. \\

\begin{figure}[t]
  \centering
  \includegraphics[width=1.0\linewidth]{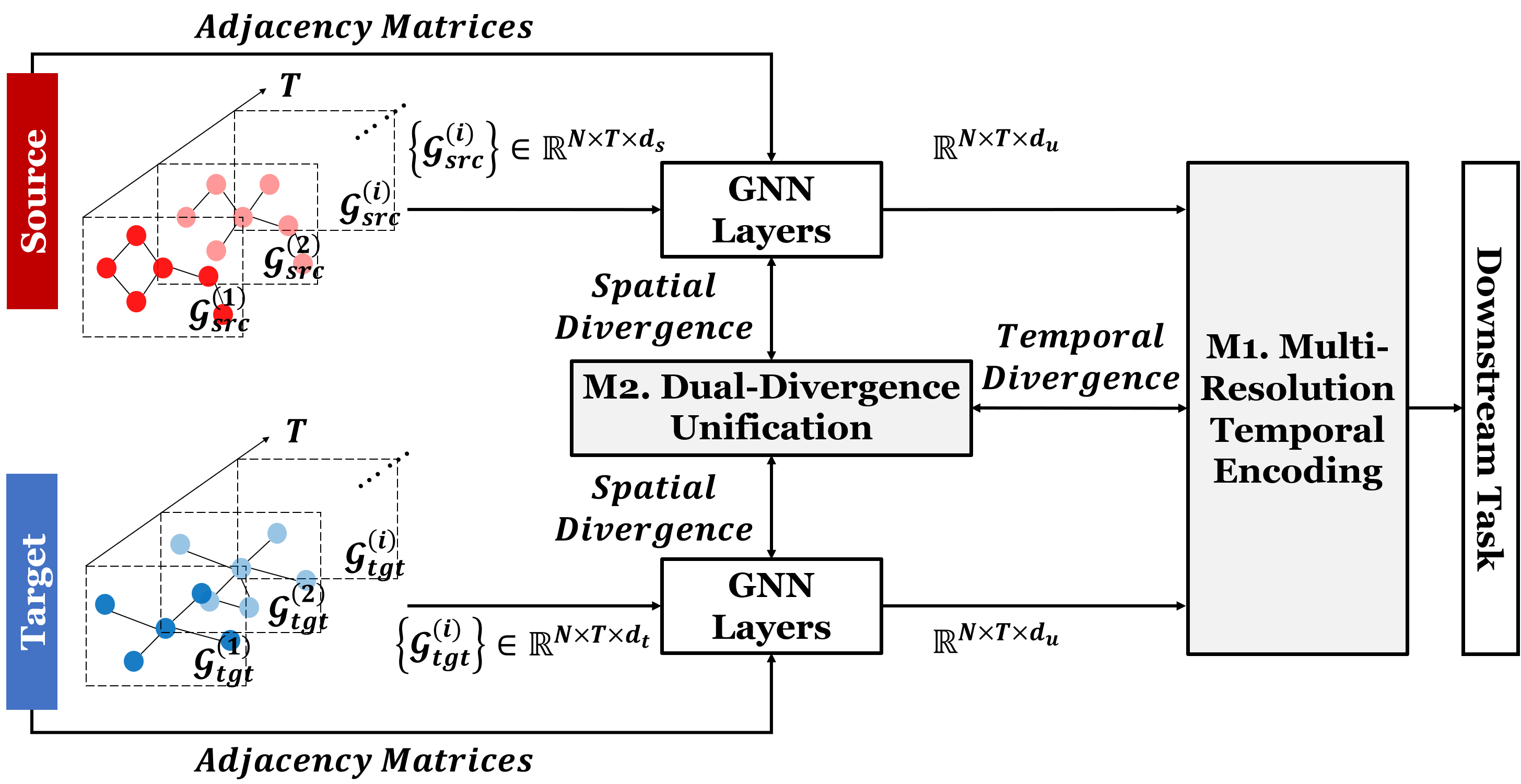}
  \caption{The proposed \fname\ framework.}
  \label{fig:framework}
\end{figure}

\noindent\textbf{M1. Modeling Domain Evolution via  Multi-Resolution Temporal Encoding.} 
Different from~\citet{Wu22Unified} that accumulates dynamic domain discrepancy over all timestamps, our method introduces a novel perspective by specifically focusing on the dynamic domain discrepancy within selected time windows. We treat each domain's time window as an integrated entity, utilizing Transformers~\cite{Vaswani17Attention} for its significant achievements in performance and computational efficiency across a variety of sequential data tasks, particularly in natural language processing. 
However, the positional encoding used in traditional Transformer models primarily serves to distinguish the sequential order of inputs rather than actual continuous time values.   
This limitation becomes particularly evident when dealing with temporal graphs that are observed at multiple timestamps, that is, the timestamps are multi-resolution and the time gap between inputs may differ.

\begin{figure}[t]
\centering
\scalebox{1.0}{
\includegraphics[width=1.0\linewidth]{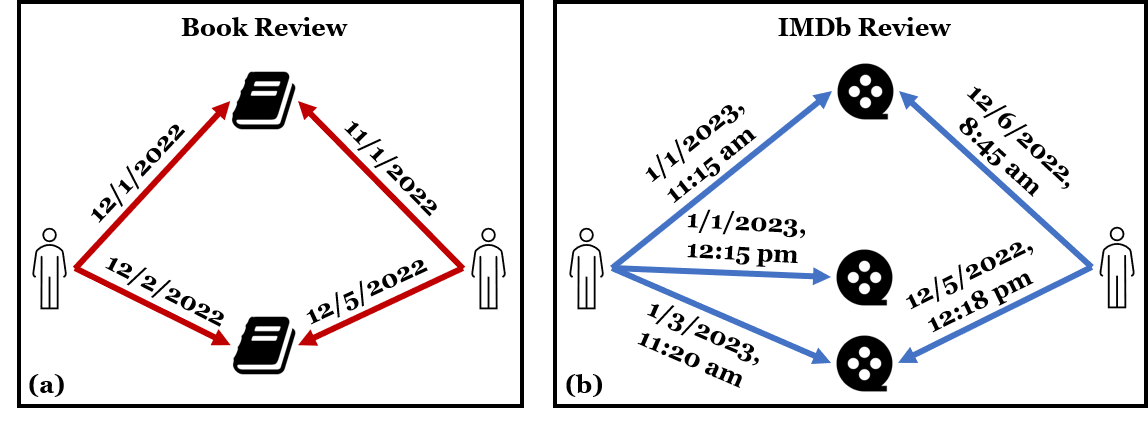}}
\caption{An illustrative example of why multi-resolution temporal encoding is important.}
\label{fig:encoding}
\end{figure}

One key challenge in temporal modeling for dynamic graphs lies in the fact that each snapshot graph is tagged with a timestamp that is both continuous and often irregular (shown in Figure~\ref{fig:encoding}). Such timestamps defy simple arithmetic operations, complicating the modeling process. To address this issue, we introduce an innovative approach of multi-resolution temporal encoding, which serves as a replacement for conventional positional encoding. This method enables our framework, \fname, to adeptly encode temporal information across multiple resolutions into a learnable representation as follows:
\begin{equation}
    \texttt{ENC} = \sum_{i=1}^T\texttt{POSITION}(\texttt{CONTEXT}(\mathcal{G}^{(i)}_{src}, \mathcal{G}^{(i)}_{tgt}))
\end{equation}
where $\texttt{CONTEXT}$ is the graph context extraction function~\cite{starnini2012random} that extracts temporal random walks from the input graphs, the $\texttt{POSITION}$ is the positional encoding function~\cite{dai2019transformer} that considers node as a token, continuous-valued timestamp as a position to capture the multi-resolution temporal information.

Next, we introduce the cross-domain self-attention layer to obtain important temporal graph representation for domain evolution. Notably, through previous operations in this framework, node embeddings of source and target sample graphs are converted into the same dimension $d_u$. Thus, a parameter-shared attention layer can be used for source and target domains to learn domain-invariant temporal node embeddings and also improve the model scalability because of its parallelism.
Specifically, for each node, we group its temporal embeddings across all the timestamps and pack them into a matrix where the order is consistent with the corresponding timestamps. 
This temporal-related matrix is passed to the self-attention layer, and the output indicates the relevance and importance of different timestamps for capturing domain evolution knowledge in terms of a specific node. 
Our cross-domain self-attention layer has advantages in two aspects:
(i) By deploying the attention layer on the source domain (target domain), we effectively capture the temporal dynamics of each domain. This allows us to model $d_{\text{GSD}}(\mathcal{G}_{src}^{(i)},\mathcal{G}_{src}^{(i+1)})$ ($d_{\text{GSD}}(\mathcal{G}_{tgt}^{(i)},\mathcal{G}_{tgt}^{(i+1)})$) in error bound.
(ii) By sharing the attention parameters of the source and target domains, we are able to capture $d_{\text{GSD}}(\mathcal{G}_{src}^{(1)}, \mathcal{G}_{tgt}^{(1)})$ in the error bounds.

\noindent\textbf{M2. Domain-Invariant Learning via Dual-Divergence Unification.}
To address the aforementioned dynamic domain divergence, we aim to learn invariant representations across evolving graphs.
Nonetheless, the process of transferring knowledge from graph-formatted data introduces inherent spatial and temporal divergences, necessitating the adoption of a dual-divergence unification approach to learning domain-invariant representations across both spatial and temporal dimensions.
In response, we present a dual-divergence unification module shown in Figure 3. 
In our implementation, we first standardize the feature dimension sizes from $d_{src}, d_{tgt}$ to a unified dimension $d_{u}$ using multi-layer perceptrons (MLPs). This standardization allows for the sharing of GNN parameters between source and target sample graphs, enhancing the learning of spatial information. We unify the MLP and the GNN into one unit named GNN Layers, and then one Gradient Reversal Layer (GRL, \citep{Ganin16Domain}) is utilized on this unit to obtain the spatial invariant representation across domains. In parallel, temporal invariance is secured by employing the GRL after the assimilation of domain evolution insights and the derivation of temporal graph representations through multi-resolution temporal encoding by module 1 (M1). The loss function $\mathcal{L}_{GRL}$ of M2 can be expressed as follows:
\begin{equation*}
\begin{aligned}
\mathcal{L}_{\text{GRL}}&=\texttt{UNIF}_{\textit{spatial}}+\texttt{UNIF}_{\textit{temporal}}\\
&=\sum_{i=1}^{T}\text{GRL}\left(\text{GNN}(\mathcal{G}_{src}^{(i)}), \text{GNN}(\mathcal{G}_{tgt}^{(i)})\right)\\
&+\sum_{i=1}^{T}\text{GRL}\left(\text{M1}(\mathcal{G}_{src}^{(i)}), \text{M1}(\mathcal{G}_{tgt}^{(i)})\right)
\end{aligned}
\end{equation*}
where $\texttt{UNIF}_{\textit{spatial}}$ and $\texttt{UNIF}_{\textit{temporal}}$ represent the spatial divergence loss on GNN Layers and the temporal divergence loss on temporal graph representation after M1, respectively. 

Overall, the objective function is defined to minimize the dual-divergence GRL loss (for all sample graphs) and the node classification loss (for source sample graphs and the few labeled nodes in target sample graphs). We detail the optimization process for \fname, as delineated in the pseudo-code provided in Algorithm~\ref{Alg} in Appendix~\ref{sec:pseudo}.

\section{Experiments}\label{sec:exp}
In this section, we evaluate the performance of \fname\ on six benchmark datasets. \fname\ exhibits superior performances compared to various state-of-the-art baselines (Section~\ref{sec:effectiveness}). Moreover, we conduct ablation studies (Section~\ref{sec:ablation}) and sensitivity analysis (Section~\ref{sec:parameter}) to demonstrate the necessity of each module in \fname\ and the reliability of \fname\ in various parameter settings. 

\subsection{Experiment Setup}
\textbf{Datasets:}
We evaluate \fname\ on our benchmark which is composed of three real-world graphs, including two graphs extracted from Digital Bibliography \& Library Project: DBLP-3 and DBLP-5~\cite{fan2021gcn}, where nodes represent authors, edges represent the co-authorship between two linked nodes; and one graph generated from human connectome project: HCP~\cite{fan2021gcn}, where nodes represent cubes of brain tissue, edges represent that two linked cubes show similar degrees of activation. Each node of these three graphs is associated with one label only. 

Our benchmark follows these principles:
(1) Dynamic: data follows the settings of graph and label evolution.
(2) Transferability: there are existing works that have explored knowledge transfer across heterogeneous domains~\cite{moon2017completely, day2017survey}. Some simple guesses are that the two graphs may have structural similarities allowing knowledge transfer~\cite{zhu2021transfer} or the attention mechanism suppressing the performance drop with heterogeneity~\cite{moon2017completely}.
To explore the potential structural similarities of the three datasets, we employ EEE-plot~\cite{prakash2010eigenspokes}, which is a scatter plot of the first three singular vectors of the adjacency matrix. In Figure~\ref{fig:spokes}, we observe there are spokes observed on the EEE-plots of three datasets, associating with the presence of well-defined communities in graphs~\cite{prakash2010eigenspokes}. The results suggest a similarity in structure across the three datasets, providing insights into the possibility of knowledge transferability among them. Our experiments also prove the validity of positive knowledge transfer.
The details of our benchmark are summarized in Table~\ref{tab:datasets}.

\begin{table}[h]
\centering
\caption{Benchmark statistics.}
\setlength{\tabcolsep}{3pt}
\scalebox{0.85}{
\begin{tabular}{ccc|ccc}
\hline
  Benchmark & Source & Target & Benchmark & Source & Target \\ \hline \hline
  1 & DBLP-5 & DBLP-3 & 4 & HCP & DBLP-5 \\
  2 & HCP & DBLP-3 & 5 & DBLP-3 & HCP \\
  3 & DBLP-3 & DBLP-5 & 6 & DBLP-5 & HCP \\
  \hline \\
\end{tabular}
}
\scalebox{0.85}{
\begin{tabular}{cccccc}
\hline
  Dataset  & \#Nodes & \#Edges & \#Attributes & \#Classes & \#Timestamps \\ \hline \hline
  DBLP-3 & 4,257 & 23,540 & 100 &  3 & 10 \\
  DBLP-5 & 6,606 & 42,815 & 100 & 5 & 10 \\  
  HCP & 5,000 & 1,955,488 & 20 & 10 & 12 \\ 
  \hline \\
\end{tabular}
}
\label{tab:datasets}
\end{table}

\begin{figure}[h]
  \centering
  \includegraphics[width=1.0\linewidth]{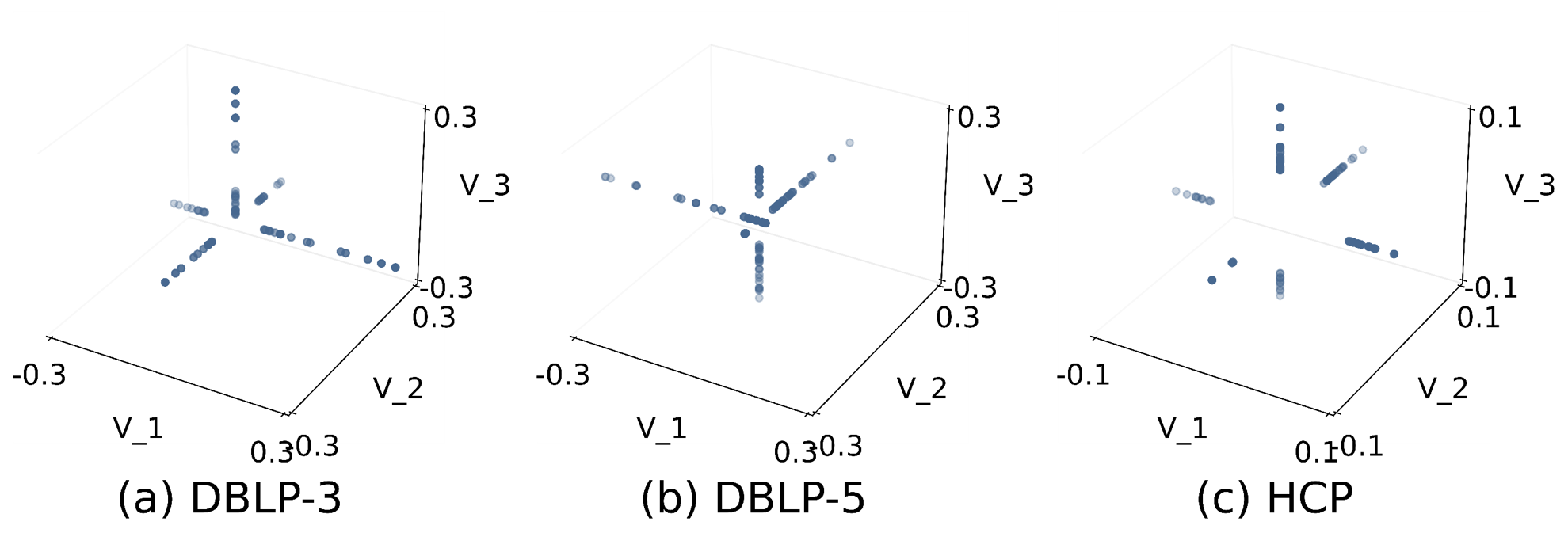}
  \caption{The EEE-plots of temporal graphs in Table~\ref{tab:datasets}.}
  \label{fig:spokes}
\end{figure}

\noindent\textbf{Baselines:} We compare \fname\ with four classical graph neural networks, four temporal graph neural networks and two graph
transfer learning methods.
\begin{desclist}[topsep=-1mm, itemsep=-1pt]
\item \underline{Classical GNNs}: Graph Convolutional Network ({GCN}, \citet{kipf17Semi}), Graph Attention Network ({GAT}, \citet{Petar18Graph}), Graph Isomorphism Network ({GIN}, \citet{Xu2019how}), Graph SAmple and aggreGatE ({GraphSAGE}, \citet{Hamilton17Advances}) are four standard graph representation benchmark architectures.
\item \underline{Temporal GNNs}: Diffusion Convolutional Recurrent Neural Network ({DCRNN}, \citet{li2017diffusion}) captures both spatial and temporal dependencies of graphs among time series. Dynamic Graph Encoder ({DyGrEncoder}, \citet{taheri2019predictive}) models embedding GNN to LSTM. Evolving Graph Convolutional Network ({EvolveGCN}, \citet{Pareja20EvolveGCN}) uses a GCN evolved by a Recurrent Neural Network (RNN) to capture the dynamism of graph sequence. Temporal Graph Convolutional
Network ({TGCN}, \citet{zhao2019t}) is a combination of GCN and the gated recurrent unit. 
\item \underline{Transfer Learning Methods}: Domain-Adversarial Neural Networks ({DANN}, \citet{Ganin16Domain}) is the first method using GRL for domain adaptation. Unsupervised Domain Adaptive Graph Convolutional Network ({UDAGCN}, \citet{wu2020unsupervised}) is a method for domain adaptation in the static graph using the attention mechanism. GRaph ADaptive Network ({GRADE}, \citet{wu2023non}) is a method for cross-network knowledge transfer from the perspective of the Weisfeiler-Lehman graph isomorphism test.
\end{desclist}

 The implementation details of the methods are provided in Appendix~\ref{sec:implement}.

\subsection{Effectiveness}\label{sec:effectiveness}

\begin{table*}[t]
\setlength{\tabcolsep}{3pt}
\centering
\caption{Comparison of different methods in node classification task using 5 labeled samples per class (area under the curve, AUC). The first four models are Classical GNN models and the next four are Temporal GNNs, we show their fine-tuned results on the target domain. The remaining three models are for transfer learning. We show results of knowledge transfer from source to target domain.}
\scalebox{0.83}{
\begin{tabular}{c|cccc|cccc|ccc|c}
\hline
 & \multicolumn{4}{c|}{Classical GNNs} & \multicolumn{4}{c|}{Temporal GNNs} & \multicolumn{3}{c|}{Transfer learning} & Ours \\
 & GCN & GAT & GIN & GraphSAGE &DCRNN & DyGrEncoder & EvolveGCN & TGCN & DANN & UDAGCN & GRADE & \fname \\ \hline
\hline
Benchmark 1 & \multirow{2}{*}{0.5609} & \multirow{2}{*}{0.5489} & \multirow{2}{*}{0.5454} & \multirow{2}{*}{0.5452} & \multirow{2}{*}{0.5637}  & \multirow{2}{*}{0.5672}    &  \multirow{2}{*}{0.5823}   &   \multirow{2}{*}{0.5640}  &   0.5416  &     0.5688   &  0.5246  & \multicolumn{1}{|c}{\textbf{0.6527}} \\
Benchmark 2 &  &  &  &  &  &  &  &  & 0.5400   &    0.5523  & 0.5223  & \multicolumn{1}{|c}{\textbf{0.6103}} \\ \hline
Benchmark 3 & \multirow{2}{*}{0.5404} & \multirow{2}{*}{0.5387} & \multirow{2}{*}{0.5422} & \multirow{2}{*}{0.5390} & \multirow{2}{*}{0.5518} & \multirow{2}{*}{0.5489}     &   \multirow{2}{*}{0.5610}  &  \multirow{2}{*}{0.5482}   &  0.5395     &    0.5660   &   0.5295    & \multicolumn{1}{|c}{\textbf{0.5915}} \\
Benchmark 4 &   &   &    &    &    &    &    &    &   0.5348    &    0.5651  &  0.5354   & \multicolumn{1}{|c}{\textbf{0.5769}}          \\ \hline
Benchmark 5 & \multirow{2}{*}{0.6756} & \multirow{2}{*}{0.6964} & \multirow{2}{*}{0.6962} & \multirow{2}{*}{0.6798} & \multirow{2}{*}{0.5710} & \multirow{2}{*}{0.6363}   &  \multirow{2}{*}{0.5679}  & \multirow{2}{*}{0.5695}    &   0.6977   &  0.7407  & 0.5170   & \multicolumn{1}{|c}{\textbf{0.7975}} \\
Benchmark 6 &   &   &   &   &   &   &   &   &    0.6981  &    0.7320   & 0.5154  & \multicolumn{1}{|c}{\textbf{0.8046}}          \\ \hline
\end{tabular}
}
\label{tab:result}
\end{table*}

\begin{table}[ht]
\setlength{\tabcolsep}{3pt}
\centering
\caption{Ablation study (AUC).}
\scalebox{0.85}{
\begin{tabular}{lccccc}
\hline \textbf{Ablation} & Benchmark 1  & Benchmark 5 & Benchmark 6\\
\hline w/o pre-training & $0.5907$ & $0.7661$& $0.7234$\\
w/o module 1  & $0.6487$ & $0.7682$ & $0.7303$ \\
w/o $\texttt{UNIF}_{\textit{spatial}}$ & $0.6367$ & $0.7939$ & $0.7985$\\
w/o $\texttt{UNIF}_{\textit{temporal}}$ & $0.6341$ & $0.7966$ & $0.8021$\\
\hline
\fname\ & $\textbf{0.6527}$ & $\textbf{0.7975}$ &  $\textbf{0.8046}$ &\\
\hline
\end{tabular}
}
\label{tab:ablation}
\end{table}

We compare \fname\ with eleven baseline methods across three real-world undirected graphs. We report the AUC of different methods on the last timestamp of the target domain in Table~\ref{tab:result}. In general, we have the following observations: 
(1) \fname\ consistently outperforms all eleven baselines on all the datasets, which demonstrates the effectiveness and generalizability of our model. Especially when adapting knowledge from DBLP-5 to DBLP-3 with five labeled samples per class, the improvement is 12.1\% compared with the second-best model (EvolveGCN). 
(2) Classical GNNs have the worst performance on four benchmarks (1, 2, 3, 4) since they can neither learn knowledge from the previous timestamps nor transfer knowledge from other domains. \fname\ boosts the performance compared with classical GNNs by up to 16.4\% (on benchmark 1).
(3) Temporal GNNs achieve second-place performance on Benchmarks 1 and 2, which means in these benchmarks, there is knowledge existing in the previous timestamps that is useful for the label prediction task in the future timestamps. Particularly, \fname\ still outperforms these temporal GNNs on Benchmarks 1 and 2 by up to 12.1\%. Notably, on Benchmarks 5 and 6, all temporal GNNs fail, while \fname\ can still has the highest performance.
(4) Transfer learning models have the second place performance on Benchmarks 5 and 6, which shows the efficacy of the domain knowledge transfer on these two benchmarks. Especially, \fname\ still does better than this kind of model on Benchmarks 5 and 6 by up to 9.9\% AUC.

\subsection{Ablation Study}\label{sec:ablation}
Considering that \fname\ consists of various components, we set up the following experiments to study the effect of different components by removing one component from \fname\ at a time: (1) removing the pre-training process; (2) removing module 1,
multi-resolution temporal encoding and attention; (3) removing module 2, the dual-divergence losses (including $\texttt{UNIF}_{\textit{spatial}}$ and $\texttt{UNIF}_{\textit{temporal}}$).
Due to the space limit, we use Benchmark 1, 5, and 6 to illustrate in this section. From Table~\ref{tab:ablation}, we have several interesting observations:
(1) Pre-training can significantly boost the model performance by up to 11.2\% (on Benchmark 6), which indicates the efficacy of knowledge transferring of our model across different graphs under the limited label setting.
(2) Module 1 achieves impressive improvement on Benchmark 5 and 6 by up to 10.2\%, which shows its strength in temporal transfer learning and also supports our theoretical analysis in Section~\ref{theory}.
(3) Both dual-divergence losses help the model better adapt knowledge from the source to the target domain. especially on Benchmark 1, the removal of $\texttt{UNIF}_{\textit{spatial}}$ ($\texttt{UNIF}_{\textit{temporal}}$) leads to a decrease in AUC by 1.6\% (1.8\%), p-value $<$ 0.001. This proves the effectiveness of dual GRLs module in alleviating the spatial and temporal divergences. 
(4) The improvements of M2 are not obvious in Benchmarks 5 and 6, and a simple guess is \fname\ variation with only M1 already achieves significant improvement than our baselines, so M2 makes less contribution to the final results.

\subsection{Parameter Sensitivity Analysis}\label{sec:parameter}
\begin{figure}[h]
\centering
\includegraphics[width=0.4\textwidth]{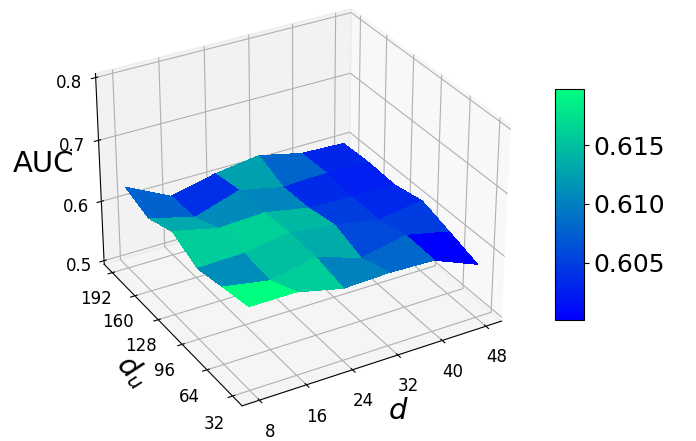}
\caption{Hyper-parameter analysis on Benchmark 2 with respect to $d_{u}$ and $d$.}
\label{fig:paramter}
\end{figure}
In this section, we study two hyper-parameters of our model: (1) the size of head dimension $d$ in M1 (modeling domain evolution via multi-resolution temporal encoding); (2) the size of the mapped features $d_{u}$ of two MLPs in M2 (domain-invariant learning via dual-divergence unification). The result is shown in Figure~\ref{fig:paramter}. Based on that, the fluctuation of the AUC (z-axis)
is less than 3\%. The AUC is slightly lower when the head dimension of module 1 becomes larger, and different values of $d$ do not affect the AUC significantly. Overall, we find \fname\ is reliable and not sensitive to the hyperparameters under study within a wide range.

\section{Related Work}\label{sec:relatedWork}
In this section, we briefly review the existing literature in the context of transfer learning and graph neural networks.  

\noindent\textbf{Transfer learning} has exhibited excellent performance in several areas, such as natural language processing~\cite{Yang21Enhancing, Ruder19Transfer}, computer vision~\cite{Zhang22Segmenting, Alhashim18High}, time series analysis~\cite{Bethge22Domain, Fawaz18Transfer}, and healthcare~\cite{panagopoulos21Transfer}. Then, several works named “continuous transfer~\cite{Wang21AFECAF, wang20transfer, Desai20Imitation}” or “dynamic domain adaptation~\cite{li19dynamic, Ke21Advances, mancini19adagraph}” are proposed to learn the evolving data. For example, \citet{Minku19Transfer} manually partitioned source data into several evolving parts and managed to solve the non-stationary source domain by performing transfer learning. There are also some works~\cite{Hoffman14Continuous, OrtizJimnez19CDOT, Liu20Learning, Wu20Continuous, Wang20Continuously, Kumar20Understanding, xie2021transfer} that addressed the scenario in which the source domain is static, and the target domain is continually evolving. Recently, \citet{Wu22Unified} modeled the knowledge transferability with dynamic source domain and dynamic target domain and defined this problem as “dynamic transfer learning.” Despite the success of dynamic transfer learning, no effort has been made to solve the problem of graph-structured data. In this paper, we aim to explore the knowledge transferability across graphs.

\noindent\textbf{Graph neural networks} capture the structure of graphs via message passing between nodes. Many significant efforts such as GCN~\cite{kipf17Semi}, GraphSAGE~\cite{Hamilton17Advances}, GAT~\cite{Petar18Graph}, GIN~\cite{Xu2019how} arose and have become indispensable baseline in a wide range of downstream tasks. Here we do not intend to provide a comprehensive survey of the wide range of GNNs. Instead, we refer the reader to excellent recent surveys to get more familiar with the topics~\cite{Wu21Comprehensive, Zhou20Graph}. Recently, several attempts have been focused on generalizing GNN from static graphs to dynamic graphs~\cite{Yu18Spatio, Michael18Modeling, Rossi2020tgn, Skarding21Foundations, Kim22DyGRAIN, You22ROLAND, Cong2023do, Yu2023Towards}. 
Specifically, \citet{Pareja20EvolveGCN} utilize common GCNs to learn node representations on each static graph snapshot and then aggregate these representations from the temporal dimension. While~\citet{Xu20Inductive} first propose to use time embedding and design a temporal graph attention layer to concatenate node, edge, and time features efficiently. 
However, the Dynamic GNN strategies often lack the capability of transferring knowledge, thus limiting their ability to leverage valuable information from other data sources. Here we further extend it to the transfer learning setting with dynamic source and target domains.

\section{Conclusion}\label{sec:conclusion}
In this paper, we investigate a novel problem named dynamic non-IID transfer learning on graphs, which intends to augment knowledge transfer from dynamic source graphs to dynamic target graphs. 
We shed light on C1 (Generalization bound) by proposing a new generalized bound in terms of historical empirical error and domain discrepancy. 
We also present \fname, an end-to-end framework with two major modules: M1. modeling domain evolution via multi-resolution temporal encoding and M2. domain-invariant learning via dual-divergence unification to alleviate evolving domain discrepancy that is specified in C2 (Computational framework). 
Extensive experiments on our carefully prepared benchmark, where \fname\ consistently outperforms leading baselines, demonstrate the efficacy of our model for dynamic non-IID transfer learning on graphs.

{\bf Reproducibility:} We have released our code and data at \url{https://github.com/wanghh7/EvoluNet}.

\nocite{langley00}

\section*{Acknowledgements}
We thank the anonymous reviewers for their constructive comments. This work is supported by 
4-VA, Cisco, Commonwealth Cyber Initiative, DARPA under the contract No. HR00112490370, Deloitte \& Touche LLP, DHS CINA, the National Science Foundation under Award No. IIS-2339989, and Virginia Tech. The views and conclusions are those of the authors and should not be interpreted as representing the official policies of the funding agencies or the government.

\section*{Impact Statement}
This paper presents work whose goal is to advance the field of Machine Learning. There are many potential societal consequences of our work, none of which we feel must be specifically highlighted here.

\bibliography{example_paper}
\bibliographystyle{icml2024}

\newpage
\appendix
\onecolumn

\section{Algorithm Analysis}\label{sec:proof}
First, we have the following assumptions from the previous work.
\begin{assumption}[$R$-Lipschitz Classifier~\cite{Wang22Understanding}]
\label{asmp:R-Lip}
Assume each classifier $h \in \mathcal{H}$ is $R$-Lipschitz in $\ell_{2}$ norm, \ie, $\forall \mathbf{x}, \mathbf{x}^{\prime} \in \mathcal{X}$,
\begin{equation*}
\left|h(\mathbf{x})-h\left(\mathbf{x}^{\prime}\right)\right| \leq R\left\|\mathbf{x}-\mathbf{x}^{\prime}\right\|_{2}.
\end{equation*}
\end{assumption}

\begin{assumption}[$\rho$-Lipschitz Loss~\cite{Wang22Understanding}]
\label{asmp:rho-Lip}
Assume the loss function $\mathcal{L}(\cdot, \cdot)$ is $\rho$-Lipschitz if $\exists~\rho>0$ such that $\forall \mathbf{x} \in \mathcal{X}$, $y, y^{\prime} \in \mathcal{Y}$ and $h, h^{\prime} \in \mathcal{H}$, the following inequalities hold:
\begin{equation*}
\begin{aligned}
    \left|\mathcal{L}\left(h^{\prime}(\mathbf{x}), y\right)-\mathcal{L}(h(\mathbf{x}), y)\right| &\leq \rho\left|h^{\prime}(\mathbf{x})-h(\mathbf{x})\right|, \\
    \left|\mathcal{L}\left(h(\mathbf{x}), y^{\prime}\right)-\mathcal{L}(h(\mathbf{x}), y)\right| &\leq \rho\left|y^{\prime}-y\right|.
\end{aligned}
\end{equation*}
\end{assumption}

\begin{assumption}[Bounded Model Complexity~\cite{Wang22Understanding, Kumar20Understanding,liang2016Statistical}]
\label{asmp:comp}
Assume the Rademachor complexity $\tilde{\Re}$ of the hypothesis class $\mathcal{H}$ is bounded, \ie, for some constant $B>0$,
$$
\tilde{\Re}(\mathcal{H})=\mathbb{E}\left[\sup _{h \in \mathcal{H}} \frac{1}{n} \sum_{i=1}^n \sigma_i h\left(\mathbf{x}_i\right)\right] \leq \frac{B}{\sqrt{n}}
$$
where $\sigma_i \sim \operatorname{Uniform}(\{-1,1\})$ for $i=1, \ldots, n$.
\end{assumption}

Next, we give the definition of dynamic Wasserstein distance on graphs, Wasserstein distance between domains, Weisfeiler-Lehman subtree, graph discrepancy, and Rademacher Complexity of hypothesis class.
\wassersteinG*

\begin{definition}[$p$-Wasserstein Distance~\cite{villani09optimal}]
\label{def:Wasserstein}
Consider two domains $\mathcal{D}_{\mu}$ and $\mathcal{D}_{\nu}$. For any $p\geq 1$, their $p$-Wasserstein distance metric is defined as:
\begin{equation*}
    W_{p}(\mathcal{D}_{\mu}, \mathcal{D}_{\nu})=\left(\inf_{\gamma\in\Gamma(\mathcal{D}_{\mu}, \mathcal{D}_{\nu})} \int d(x, y)^{p} \mathrm{~d} \gamma(x, y)\right)^{1/p},
\end{equation*}
where $\Gamma(\mathcal{D}_{\mu}, \mathcal{D}_{\nu})$ is the set of all measures over $\mathcal{D}_{\mu}\times\mathcal{D}_{\nu}$.
\end{definition}

\begin{definition}[Weisfeiler-Lehman subtree~\cite{Shervashidze11Weisfeiler}]
\label{def:WLsubtree}
Given a graph $\mathcal{G} = (\mathcal{V}, \mathcal{E})$, the Weisfeiler-Lehman subtree of depth $m$ rooted at $\mathbf{v} \in \mathcal{V}$ can be defined as:
\begin{equation*}
    f_m(\mathbf{v}) = f_m \left( f_{m-1}(\mathbf{v}); \cup_{\mathbf{u}\in\mathcal{N} (\mathbf{v})}f_{m-1}(\mathbf{u})\right),
\end{equation*}
where $f_0(\mathbf{v})$ is the initial node attributes for node $\mathbf{v}$, $f_i, i=1,\cdots, m,\cdots$ denotes the labeling function, $\mathcal{N} (\mathbf{v})$ denotes the neighbors of node $\mathbf{v}$.
\end{definition}

\begin{definition}[Graph Discrepancy~\cite{wu2023non}]
\label{def:GraphDis}
Given two graphs $\mathcal{G}_{\mu}$ and $\mathcal{G}_{\nu}$, the graph discrepancy between the two graphs can be represented as:
\begin{equation*}
    d_{\text{GSD}}(\mathcal{G}_{\mu}, \mathcal{G}_{\nu})=\lim _{M \rightarrow \infty} \frac{1}{M+1} \sum_{m=0}^M d_b(\mathcal{G}_{\mu}^m, \mathcal{G}_{\nu}^m),
\end{equation*}
where $\mathcal{G}^m$ is the Weisfeiler-Lehman subgraph at depth $m$ for an input graph $\mathcal{G}$, $d_b(\cdot,\cdot)$ is the base domain discrepancy, here we use the $p$-Wasserstein distance metric $W_{p}$.
\end{definition}

\begin{definition}[Rademacher Complexity~\cite{Bartlett02Rademacher}]
\label{def:Rademacher}
Given a sample $S= (\mathbf{X}_1, \cdots, \mathbf{X}_N) \in \mathcal{X}^{N}$, the empirical Rademacher complexity of $\mathcal{H}$ given $S$ is defined as:
\begin{equation*}
    \hat{\Re}_{S}(\mathcal{H})=\mathbb{E}_{\boldsymbol{\sigma}}\left[\sup _{h \in \mathcal{H}} \sum_{i=1}^{N} \sigma_{i} h\left(\mathbf{x}_{i}\right) \mid S=\left(\mathbf{x}_{1}, \cdots, \mathbf{x}_{N}\right)\right],
\end{equation*}
where $\boldsymbol{\sigma}=\left(\sigma_{1}, \cdots, \sigma_{m}\right)$ is a vector of independent random variables from the Rademacher distribution.
\end{definition}

Then, we use Lemma~\ref{lemma:shift domains} to bound the error difference between arbitrary two domains and use Lemma~\ref{lemma:stability} to bound the difference between empirical and expected errors.
\shiftDomains*

\stability*
The proof of Lemma~\ref{lemma:stability} can be found in the proof of Proposition 1 of the proof of Wang \etal~\cite{Wang22Understanding} and Lemma A.1 of Kumar \etal~\cite{Kumar20Understanding}.

\begin{lemma}[McDiarmid's inequality]
\label{lemma:McDiarmid}
Let function $f$ satisfies for all $1\leq i\leq N$, and all $\mathbf{X}_{1}, \cdots, \mathbf{X}_{N}, \mathbf{X}_{i}^{\prime} \in \mathcal{X}$,
\begin{equation}
\left|f\left(\mathbf{X}_{1}, \cdots, \mathbf{X}_{i}, \cdots, \mathbf{X}_{N}\right)-f\left(\mathbf{X}_{1}, \cdots, \mathbf{X}_{i}^{\prime}, \cdots, \mathbf{X}_{N}\right)\right| \leq c_{i},
\end{equation}
where bound $c_1, \cdots, c_N$ are constants. Then, for any  $\epsilon>0$,
\begin{equation}
\operatorname{Pr}[f-\mathbb{E}[f] \geq \epsilon] \leq \exp \left(\frac{-2 \epsilon^{2}}{\sum_{i=1}^{N} c_{i}^{2}}\right).
\end{equation}
\end{lemma}

Based on the above conclusion, Theorem \ref{THM:errorBound} and its proof are given as follows.
\errorBound*

\begin{proof}
For the sake of simplicity here, we use $\mathcal{G}_{src}^{(i)}$ and $\mathcal{G}_{tgt}^{(i)}$ be the Weisfeiler-Lehman subgraphs of source domain and the target domain at $i^{th}$ timestamp, following the discussion of Wu \etal~\yrcite{wu2023non}, the representations can be considered as 
\textit{conditionally independent} with respect to Weisfeiler-Lehman subgraph. $\mathcal{B}\in(\mathcal{G} \times \mathcal{Y})^{\tilde{n}} $ is the measurable subset over $\mathcal{G}_{src}^{(1)}\times \cdots \times \mathcal{G}_{src}^{(T)}\times \mathcal{G}_{src}^{(1)}\times \cdots \times \mathcal{G}_{tgt}^{(T)}$, and we define a function $g$ over $\mathcal{B}$ as follows~\cite{Wu22Unified}:
\begin{equation}\label{equ:defG_B}
    g(\mathcal{B}) = \sup _{h \in \mathcal{H}} \epsilon_{tgt}^{(T+1)}(h)-\frac{1}{2T} \sum_{i = 1}^{T}\left(\hat{\epsilon}_{src}^{(i)}(h)+\hat{\epsilon}_{tgt}^{(i)}(h)\right),
\end{equation}
where $\hat{\epsilon}_{src}^{(i)}(h)=\frac{1}{N_{src}^{(i)}}\sum_{j=1}^{N_{src}^{(i)}}[\mathcal{L}(h(\mathbf{x}_j), y_j)]$ ($\mathbf{x}_j$ is the feature of $j^\text{th}$ sample in  $\mathbf{X}_{src}^{(i)}$) and $\hat{\epsilon}_{tgt}^{(i)}(h)=\frac{1}{N_{tgt}^{(i)}}\sum_{j=1}^{N_{tgt}^{(i)}}[\mathcal{L}(h(\mathbf{x}_j), y_j)]$ ($\mathbf{x}_j$ is the feature of $j^\text{th}$ sample in  $\mathbf{X}_{tgt}^{(i)}$) are the estimate errors on graph $\mathcal{G}_{src}^{(i)}$ and $\mathcal{G}_{tgt}^{(i)}$. Let $\mathcal{B}$ and $\mathcal{B}^{\prime}$ be two measurable subsets containing only one different source sample in $\mathcal{G}_{src}^{(i)}$, then we have 
\begin{equation*}
    \left|g(\mathcal{B})-g\left(\mathcal{B}^{\prime}\right)\right| \leq \frac{2 \rho}{2 N_{tgt}^{(i)} T} \leq \frac{\rho}{\tilde{n} T}.
\end{equation*}
The same result holds for different target samples. Based on McDiarmid's inequality (see Lemma~\ref{lemma:McDiarmid}), we have for any $\epsilon>0$,
\begin{equation*}
    \operatorname{Pr}\left[g(\mathcal{B})-\mathbb{E}_{\mathcal{B}}[g(\mathcal{B})] \geq \epsilon\right] \leq \exp \left(\frac{-2 \tilde{n} T^{2} \epsilon^{2}}{\rho^{2}}\right).
\end{equation*}
Then, for any $\delta>0$, with probability at least $1-\delta$, the following holds:
\begin{equation*}
    g(\mathcal{B}) \leq \mathbb{E}_{\mathcal{B}}[g(\mathcal{B})]+\frac{\rho}{T} \sqrt{\frac{\log \frac{1}{\delta}}{2 \tilde{n}}}.
\end{equation*}

In addition, Definition~\ref{def:GraphDis} gives a metric to measure the graph discrepancy based on $p$-Wasserstein distance $W_{p}$, so we can generalize Lemma~\ref{lemma:shift domains} to graphs. It bounds the population error difference of a classifier between a pair of shifted domains on graphs. For any $h \in \mathcal{H}$ and any $i \in \{1, \cdots, T\}$, we have
\begin{equation*}
\begin{aligned}
    \epsilon_{tgt}^{(i)}(h) &=\epsilon_{src}^{(i)}(h)+\epsilon_{tgt}^{(i)}(h)-\epsilon_{src}^{(i)}(h),\\
    & \leq \epsilon_{src}^{(i)} + \rho\sqrt{R^2+1}d_{\text{GSD}}(\mathcal{G}_{tgt}^{(i)}, \mathcal{G}_{src}^{(i)}).
\end{aligned} 
\end{equation*}
Similarly, we have
\begin{equation*}
\begin{aligned}
    \epsilon_{tgt}^{(T+1)}(h) &=\epsilon_{tgt}^{(i)}(h)+\epsilon_{tgt}^{(T+1)}(h)-\epsilon_{tgt}^{(i)}(h),\\
    & \leq \epsilon_{tgt}^{(i)} + \rho\sqrt{R^2+1}d_{\text{GSD}}(\mathcal{G}_{tgt}^{(T+1)}, \mathcal{G}_{tgt}^{(i)}).
\end{aligned} 
\end{equation*}
Then, we have
\begin{equation*}
\begin{aligned}
    &\sum_{i=1}^{T}\left(\epsilon_{tgt}^{(T+1)}(h)-\epsilon_{tgt}^{(i)}(h)\right)\\
    =&\epsilon_{tgt}^{(T+1)}(h)-\epsilon_{tgt}^{(T)}(h)+\cdots+\epsilon_{tgt}^{(2)}(h)-\epsilon_{tgt}^{(1)}(h)+\sum_{i=2}^{T}\left(\epsilon_{tgt}^{(T+1)}(h)-\epsilon_{tgt}^{(i)}(h)\right)\\
    \leq&\rho\sqrt{R^2+1}\left(d_{\text{GSD}}(\mathcal{G}_{tgt}^{(T)}, \mathcal{G}_{tgt}^{(T+1)})+\cdots+d_{\text{GSD}}(\mathcal{G}_{tgt}^{(1)}, \mathcal{G}_{tgt}^{(2)})\right)
    +\sum_{i=2}^{T}\left(\epsilon_{tgt}^{(T+1)}(h)-\epsilon_{tgt}^{(i)}(h)\right)\\
    \leq&T\tilde{W}_{p}+\sum_{i=2}^{T}\left(\epsilon_{tgt}^{(T+1)}(h)-\epsilon_{tgt}^{(i)}(h)\right)\leq\frac{T(T+1)}{2}\tilde{W}_{p}
\end{aligned}
\end{equation*}

Then
\begin{equation*}
\begin{aligned}
    &\mathbb{E}_{\mathcal{B}}[g(\mathcal{B})]\\
    =& \mathbb{E}_{\mathcal{B}}\left[\sup_{h \in \mathcal{H}}\epsilon_{tgt}^{(T+1)}(h)-\frac{1}{2 T} \sum_{i=1}^{T}\left(\hat{\epsilon}_{src}^{(i)}(h)+\hat{\epsilon}_{tgt}^{(i)}(h)\right)\right] \\
    =& \mathbb{E}_{\mathcal{B}}\left[\sup _{h \in \mathcal{H}} \epsilon_{tgt}^{(T+1)}(h)-\frac{1}{2T} \sum_{i=1}^{T}\left(\epsilon_{src}^{(i)}(h)+\epsilon_{tgt}^{(i)}(h)\right)+\frac{1}{2T} \sum_{i=1}^{T}\left(\epsilon_{src}^{(i)}(h)-\hat{\epsilon}_{src}^{(i)}(h)\right)+\frac{1}{2T} \sum_{i=1}^{T}\left(\epsilon_{tgt}^{(i)}(h)-\hat{\epsilon}_{tgt}^{(i)}(h)\right)\right] \\
    =& \frac{1}{2T} \sup _{h \in \mathcal{H}}\left(\sum_{i=1}^{T}\left(\epsilon_{tgt}^{(T+1)}(h)-\epsilon_{tgt}^{(i)}(h)\right)+\sum_{i=1}^{T}\left(\epsilon_{tgt}^{(T+1)}(h)-\epsilon_{src}^{(i)}(h)\right)\right) \\
    +&\mathbb{E}_{\mathcal{B}}\left[\sup _{h \in \mathcal{H}} \frac{1}{2T} \sum_{i=1}^{T}\left(\epsilon_{src}^{(i)}(h)-\hat{\epsilon}_{src}^{(i)}(h)\right)+\frac{1}{2T} \sum_{i=1}^{T}\left(\epsilon_{tgt}^{(i)}(h)-\hat{\epsilon}_{tgt}^{(i)}(h)\right)\right] \\
    \leq & \frac{1}{2T} \sup _{h \in \mathcal{H}}\left(\sum_{i=1}^{T}\left(\epsilon_{tgt}^{(T+1)}(h)-\epsilon_{tgt}^{(i)}(h)\right)+\sum_{i=1}^{T}\left(\epsilon_{tgt}^{(T+1)}(h)-\epsilon^{(i)}_{tgt}(h)\right)+\sum_{i=1}^{T}\left(\epsilon^{(i)}_{tgt}(h)-\epsilon^{(i)}_{src}(h)\right)\right) \\
    +&\mathbb{E}_{\mathcal{B}}\left[\frac{1}{2T} \sum_{i=1}^{T} \sup _{h \in \mathcal{H}}\left(\epsilon_{src}^{(i)}(h)-\hat{\epsilon}_{src}^{(i)}(h)\right)+\frac{1}{2T} \sum_{i=1}^{T} \sup _{h \in \mathcal{H}}\left(\epsilon_{tgt}^{(i)}(h)-\hat{\epsilon}_{tgt}^{(i)}(h)\right)\right]\\
    \leq&\frac{1}{2T}\left[\frac{T(T+1)}{2}\tilde{W}_{p}+\frac{T(T+1)}{2}\tilde{W}_{p}+T\tilde{W}_{p}\right]    +\mathbb{E}_{\mathcal{B}}\left[\frac{1}{2T}\sum_{i=1}^T\tilde{\Re}_{\mathcal{D}_{src}^{(i)}}(\mathcal{H}_{\mathcal{L}})+\frac{1}{2T}\sum_{i=1}^T\tilde{\Re}_{\mathcal{D}_{tgt}^{(i)}}(\mathcal{H}_{\mathcal{L}})\right]\\
    \leq&\frac{T+2}{2}\tilde{W}_{p}+\tilde{\Re}(\mathcal{H}_{\mathcal{L}}).
\end{aligned}
\end{equation*}

According to (\ref{equ:defG_B}), we have for any $h \in \mathcal{H}$,
\begin{equation}
\label{equ:boundSum}
    \epsilon_{tgt}^{(T+1)}(h) \leq \frac{1}{2T} \sum_{i=1}^{T}\left(\hat{\epsilon}_{src}^{(i)}(h)+\hat{\epsilon}_{tgt}^{(i)}(h)\right)+\mathbb{E}_{\mathcal{B}}[g(\mathcal{B})]+\frac{\rho}{T} \sqrt{\frac{\log \frac{1}{\delta}}{2 \tilde{n}}}.
\end{equation}
W.l.o.g., we assume $\hat{\epsilon}_{src}^{(1)}\leq\hat{\epsilon}_{src}^{(2)}\leq\cdots\leq\hat{\epsilon}_{src}^{(T)}$ for simplify. Consider the last term in $\sum_{i=1}^{T}\left(\hat{\epsilon}_{src}^{(i)}(h)\right)$, for some constant $B>0$,
\begin{equation}
\begin{aligned}
    (lemma~\ref{lemma:stability})\hat{\epsilon}_{src}^{(T)}\leq&\epsilon_{src}^{(T)}+\mathcal{O}\left(\frac{\rho B}{\sqrt{\tilde{n}}}+\sqrt{\frac{\log\frac{1}{\delta}}{\tilde{n}}}\right)\\
    (lemma~\ref{lemma:shift domains})\leq&\epsilon_{src}^{(T-1)}+\rho\sqrt{R^2+1}d_{\text{GSD}}(\mathcal{G}_{src}^{(T)}, \mathcal{G}_{src}^{(T-1)})
    +\mathcal{O}\left(\frac{\rho B}{\sqrt{\tilde{n}}}+\sqrt{\frac{\log\frac{1}{\delta}}{\tilde{n}}}\right)\\
    \leq&\cdots\\
    \leq&\epsilon_{src}^{(1)}+(T-1)\tilde{W}_{p}+\mathcal{O}\left(\frac{\rho B}{\sqrt{\tilde{n}}}+\sqrt{\frac{\log\frac{1}{\delta}}{\tilde{n}}}\right)\\
    (lemma~\ref{lemma:stability})\leq&\hat{\epsilon}_{src}^{(1)}+(T-1)\tilde{W}_{p}+\mathcal{O}\left(\frac{\rho B}{\sqrt{\tilde{n}}}+\sqrt{\frac{\log\frac{1}{\delta}}{\tilde{n}}}\right).
\end{aligned}
\label{equ:lastTerm}
\end{equation}
For the second last term in $\sum_{i=1}^{T}\left(\hat{\epsilon}_{src}^{(i)}(h)\right)$, we have
\begin{equation*}
\begin{aligned}
    (lemma~\ref{lemma:stability}) \hat{\epsilon}_{src}^{(T-1)}\leq&\epsilon_{src}^{(T-1)}+\mathcal{O}\left(\frac{\rho B}{\sqrt{\tilde{n}}}+\sqrt{\frac{\log (1 / \delta)}{\tilde{n}}}\right)\\
    \leq&\epsilon_{src}^{(T)}+\mathcal{O}\left(\frac{\rho B}{\sqrt{\tilde{n}}}+\sqrt{\frac{\log\frac{1}{\delta}}{\tilde{n}}}\right)\\
    (Eq.\ref{equ:lastTerm})\leq&\hat{\epsilon}_{src}^{(1)}+(T-1)\tilde{W}_{p}+\mathcal{O}\left(\frac{\rho B}{\sqrt{\tilde{n}}}+\sqrt{\frac{\log\frac{1}{\delta}}{\tilde{n}}}\right).
\end{aligned}
\end{equation*}
It is easy to see that this can be bounded for source or target across time. Generally,
\begin{equation*}
\begin{aligned}
    \frac{1}{2T}\sum_{i=1}^{T}\left(\hat{\epsilon}_{src}^{(i)}(h)\right)\leq\frac{1}{2}\min_{1\leq i\leq T}(\hat{\epsilon}_{src}^{(i)})+\frac{T-1}{2}\tilde{W}_{p}+\mathcal{O}\left(\frac{\rho B}{\sqrt{\tilde{n}}}+\sqrt{\frac{\log\frac{1}{\delta}}{\tilde{n}}}\right),\\
    \frac{1}{2T}\sum_{i=1}^{T}\left(\hat{\epsilon}_{tgt}^{(i)}(h)\right)\leq\frac{1}{2}\min_{1\leq i\leq T}(\hat{\epsilon}_{tgt}^{(i)})+\frac{T-1}{2}\tilde{W}_{p}+\mathcal{O}\left(\frac{\rho B}{\sqrt{\tilde{n}}}+\sqrt{\frac{\log\frac{1}{\delta}}{\tilde{n}}}\right).
\end{aligned}
\end{equation*}
Therefore, from (\ref{equ:boundSum}), we have
\begin{equation}
     \epsilon_{tgt}^{(T+1)}(h) \leq \frac{1}{2}\min_{1\leq i\leq T}\left(\hat{\epsilon}_{src}^{(i)}(h)+\hat{\epsilon}_{tgt}^{(i)}(h)\right)+\frac{3T}{2}\tilde{W}_{p}+\tilde{\Re}(\mathcal{H}_{\mathcal{L}})+\mathcal{O}\left(\frac{\rho B}{\sqrt{\tilde{n}}}+\sqrt{\frac{\log\frac{1}{\delta}}{\tilde{n}}}\right).
\end{equation}
which completes the proof.
\end{proof}

\section{Optimization and Pseudo Code}\label{sec:pseudo}
The goal of the training process is to minimize the dual-divergence GRL loss (for all sample graphs) and the node classification loss (for source sample graphs and the few labeled nodes in target sample graphs). The overall loss function can be written as follows:
\begin{equation}\label{equ:loss}
\mathcal{L}_{total}=\mathcal{L}_{\text{GRL}}+\gamma_1*\mathcal{L}_{task}
\end{equation}
where $\mathcal{L}_{GRL}$ represents the dual GRL loss, $\mathcal{L}_{task}$ represents the loss for classification on labeled nodes, and the hyperparameter $\gamma_1$ balances the contribution of the two terms. In the paper, we consider the node classification task, $\mathcal{L}_{task}$ is therefore defined as follows:
\begin{equation}
\begin{aligned}
    \mathcal{L}_{task}&=\mathcal{L}_{source}+\mathcal{L}_{target}\\
    &=\sum_{i=1}^{T}\mathcal{L}_{\text{CE}}\left(h(\mathcal{G}_{src}^{(i)}), \mathcal{Y}_{src}^{(i)}\right)+\gamma_2*\sum_{i=1}^{T+1}\mathcal{L}_{\text{CE}}\left(h(\tilde{\mathcal{G}}_{tgt}^{(i)}), \tilde{\mathcal{Y}}_{tgt}^{(i)}\right)
\end{aligned}    
\end{equation}
where $h(\cdot)$ is the classifier for the downstream task, $\mathcal{L}_{source}$ and $\mathcal{L}_{target}$ represent the node classification loss on the source and target domains, here we employ cross-entropy loss $\mathcal{L}_{\text{CE}}$, and the contribution of the two terms is balanced by $\gamma_2$. 

We provide the pseudo-code of \fname\ in Algorithm~\ref{Alg} and we employ Adam~\cite{Kingma15Adam} as the optimizer. Given a set of source sample graphs $\{\mathcal{G}_{src}^{(i)} = (\mathcal{V}_{src}^{(i)}, \mathcal{E}_{src}^{(i)})\}_{i=1}^T$ with rich label information $\{\mathcal{Y}_{src}^{(i)}\}_{i=1}^T$, and a set of target graphs $\{\mathcal{G}_{tgt}^{(i)} = (\mathcal{V}_{tgt}^{(i)}, \mathcal{E}_{tgt}^{(i)})\}_{i=1}^{T+1}$ with few label information $\{\tilde{\mathcal{Y}}_{tgt}^{(i)}\}_{i=1}^{T+1}$, our proposed \fname\ framework aims to predict $\hat{\mathcal{Y}}_{tgt}^{(T+1)}$ in the latest target sample graph $\mathcal{G}_{tgt}^{(T+1)}$. 
We initialize each of the models and the classifier in Step 1. Steps 2-7 correspond to the pre-train process: in Step 3, we map sample graphs from source and target domains to a shared latent space using two separate MLPs; then the mapped representations are passed to a domain-invariant GNN for computing domain-invariant spatial representations in Step 4; followed by a domain-invariant module 1 for computing domain-invariant temporal graph representations in Step 5; while in Step 6, models are trained by minimizing the objective function. In Steps 8-10, we fine-tune the MLP of the target domain, the domain-invariant GNN, the domain-invariant module 1, and the classifier $h(\cdot)$ on the latest target domain $\mathcal{G}_{tgt}^{(T+1)}$.

\begin{algorithm}[htbp]
\caption{The \fname\ Learning Framework.}
\label{Alg}
\begin{algorithmic}[1]
\REQUIRE ~~\\
    (i) a set of source sample graphs $\{\mathcal{G}_{src}^{(i)} = (\mathcal{V}_{src}^{(i)}, \mathcal{E}_{src}^{(i)})\}_{i=1}^T$ with rich label information $\{\mathcal{Y}_{src}^{(i)}\}_{i=1}^T$;
    (ii) a set of target sample graphs $\{\mathcal{G}_{tgt}^{(i)} = (\mathcal{V}_{tgt}^{(i)}, \mathcal{E}_{tgt}^{(i)})\}_{i=1}^{T+1}$ with few label information $\{\tilde{\mathcal{Y}}_{tgt}^{(i)}\}_{i=1}^{T+1}$.
\ENSURE ~~\\
    Prediction $\hat{\mathcal{Y}}_{tgt}^{(T+1)}$ of unlabeled examples in $\mathcal{G}_{tgt}^{(T+1)}$.\\
    \STATE\label{step:init} Initialize two MLPs for source and target, the domain-invariant GNN, the domain-invariant module 1, the dual-divergence unification module, and the classifier $h(\cdot)$ for the downstream task in $\mathcal{G}_{tgt}^{(T+1)}$.
    \WHILE{not converge}\label{step:startPretrain}
        \STATE\label{step:MLP} Compute representations in a shared latent space of both $\{\mathcal{G}_{src}^{(i)}\}_{i=1}^T$ and $\{\mathcal{G}_{tgt}^{(i)}\}_{i=1}^{T+1}$ via two MLPs.
        \STATE\label{step:GNN} Compute domain-invariant spatial representations of both $\{\mathcal{G}_{src}^{(i)}\}_{i=1}^T$ and $\{\mathcal{G}_{tgt}^{(i)}\}_{i=1}^{T+1}$ via the domain-invariant GNN and first GRL.
        \STATE\label{step:transformer} Compute domain-invariant temporal graph representations of both $\{\mathcal{G}_{src}^{(i)}\}_{i=1}^T$ and $\{\mathcal{G}_{tgt}^{(i)}\}_{i=1}^{T+1}$ via the domain-invariant module 1 and second GRL.
        \STATE\label{step:update} Update the hidden parameters of two MLPs, the GNN, module 1, and the dual-divergence unification module by minimizing the loss function in Eq.~\ref{equ:loss}.
    \ENDWHILE\label{step:endPretrain}
    \WHILE{not converge}\label{step:startFinetune}
        \STATE Fine-tune MLP for the target domain, the GNN, module 1, and the classifier $h(\cdot)$ for the downstream task.
    \ENDWHILE\label{step:endFinetune}
\end{algorithmic}
\end{algorithm}

\section{Implementation Details}\label{sec:implement}
We compare \fname\ with four classical graph neural networks GCN~\cite{kipf17Semi}, GAT~\cite{Petar18Graph}, GIN~\cite{Xu2019how}, GraphSAGE~\cite{Hamilton17Advances}; four temporal graph neural networks DCRNN~\cite{li2017diffusion}, DyGrEncoder~\cite{taheri2019predictive}, EvolveGCN~\cite{Pareja20EvolveGCN}, TGCN~\cite{zhao2019t}); and three graph transfer learning methods DANN~\cite{Ganin16Domain}, UDAGCN~\cite{wu2020unsupervised}, GRADE~\cite{wu2023non}).
For a fair comparison, the output dimensions of all GNNs including baselines and \fname\ are set to 16. We conduct experiments with only five labeled samples in each class of the target dataset and test model performance based on all the rest of the unlabeled nodes. For non-temporal GNNs, since they cannot process dynamic graphs directly, we train each model on the graph of the last timestamp. Specifically, for classical GNNs, they are trained on the target dataset
for 1000 epochs; for transfer learning models, after training on the source dataset for 2000 epochs, they are fine-tuned on the target dataset for 600 epochs. We use GCN as the feature extractor of DANN and follow the instructions from the original paper of UDAGCN~\cite{wu2020unsupervised} to build a union set for input features between the source and target domains by setting zeros for unshared features. The original code for GRADE does not support cross-domain transfer with different feature and class dimensions; we processed the features with a linear layer and constructed a joint label space.
For four temporal GNNs, they are trained using all timestamps of the target dataset for 1000 epochs.

For \fname, it is firstly pre-trained for 2000 epochs, then fine-tuned on the target dataset for 600 epochs using limited labeled data in each class. Since the label of each node in current benchmarks is consistent in every timestamp, in this paper, the output of module 1 in \fname\ is aggregated using the average function over all the timestamps, but our model can easily be applied to the settings where labels of each node are changed in different timestamps by simply removing the aggregation operation. We use Adam optimizer with learning rate 3e-3. Considering the imbalanced label distribution, the area under the receiver of the characteristic curve (AUC) is used as the evaluation metric. We run all the experiments with 25 random seeds. The experiments are performed on a Ubuntu20 machine with 16 3.8GHz AMD Cores and a single 24GB NVIDIA GeForce RTX3090.


\end{document}